%% file: main_arxiv.tex
\documentclass{article}

 \usepackage[preprint]{neurips_2025}

\usepackage{hyperref}
\usepackage{xcolor}
\usepackage{algorithmic}
\usepackage{longtable}
\usepackage{wrapfig}

\newcounter{excounter}[section]
\renewcommand{\theexcounter}{\thesection.\arabic{excounter}}

\usepackage[utf8]{inputenc} 
\usepackage[T1]{fontenc}    
\usepackage{hyperref}       
\usepackage{url}            
\usepackage{booktabs}       
\usepackage{amsfonts}       
\usepackage{nicefrac}       
\usepackage{microtype}      
\usepackage{subcaption} 
\usepackage[symbol]{footmisc}
\usepackage[most]{tcolorbox}
\tcbset {
  base/.style={
    arc=0mm, 
    bottomtitle=0.5mm,
    boxrule=0mm,
    colbacktitle=!10!white, 
    coltitle=black, 
    fonttitle=\bfseries, 
    left=2.5mm,
    leftrule=1mm,
    right=3.5mm,
    title={#1},
    toptitle=0.75mm, 
  }
}

\tcbuselibrary{listings,breakable}
\usepackage{verbatimbox}
\usepackage{listings}

\hypersetup{
    colorlinks,
    linkcolor={red!50!black},
    citecolor={blue!60!black},
    urlcolor={blue!80!black}
}

\newcommand\blfootnote[1]{%
  \begingroup
  \renewcommand\thefootnote{}\footnote{#1}%
  \addtocounter{footnote}{-1}%
  \endgroup
}
\colorlet{darkred}{red!60!black}

\usepackage{algorithm}
\usepackage{enumitem}
  \setlist{leftmargin=*}
\usepackage{graphicx}
  \graphicspath{ {./figs/} }
\usepackage{lineno}
  \linenumbers
\usepackage{YK}
\usepackage{wrapfig}  
\input{macros}

\title{CARROT:  A Cost Aware Rate Optimal Router}

\author{%
  Seamus Somerstep$^{\dagger \square}$ \ \ \
  Felipe Maia Polo$^{\dagger  \square}$ \ \ \
  Allysson Flavio Melo de Oliveira $^ {\ddag \circ}$ \\ 
  \textbf{Prattyush Mangal} $^ \ddag$ \ \ \
  \textbf{Mírian Silva} $^ {\ddag \circ \triangle}$  \ \ \
  \textbf{Onkar Bhardwaj} $^ {\ddag \circ}$ \\ \vspace{0.5cm}
  \textbf{Mikhail Yurochkin}* $^ {\ddag \circ}$ \ \ \
  \textbf{Subha Maity}* $^ \spadesuit$ \\ 
\normalsize $^\square$ Department of Statistics, University of Michigan\\
$^\ddagger$ IBM Research  ~~~ $^\circ$ MIT-IBM Watson AI Lab  ~~~
$\triangle$ Federal University of Minas Gerais\\
$^ \spadesuit$ Department of Statistics and Actuarial Science, University of Waterloo 
}

\begin{document}
\nolinenumbers

\maketitle




\begin{center}
   \large{ 
   \href{https://huggingface.co/CARROT-LLM-Routing}{\textcolor{purple!80!black}{\texttt{huggingface.co/CARROT-LLM-Routing}}}
   }
    \end{center}
  %


\begin{abstract}

\input{sections_arxiv/abstract}
\end{abstract}


\input{sections_arxiv/intro_new}

\input{sections_arxiv/setup}
\input{sections_arxiv/lower-bound}
\input{sections_arxiv/experiment}

\input{sections_arxiv/conclusion}

\bibliography{seamus,YK,RAGBench}
\bibliographystyle{abbrvnat}

\newpage
\appendix

\input{sections_arxiv/supp}

\end{document}

%% file: macros.tex
\usepackage{multirow}

\newcommand{\newdata}{\text{SPROUT}}
\usepackage{listings}

\usepackage{xcolor}

\definecolor{codegreen}{rgb}{0,0.6,0}
\definecolor{codegray}{rgb}{0.5,0.5,0.5}
\definecolor{codepurple}{rgb}{0.58,0,0.82}
\definecolor{backcolour}{rgb}{0.95,0.95,0.92}

\lstdefinestyle{mystyle}{
    backgroundcolor=\color{backcolour},   
    commentstyle=\color{codegreen},
    keywordstyle=\color{magenta},
    numberstyle=\tiny\color{codegray},
    stringstyle=\color{codepurple},
    basicstyle=\ttfamily\footnotesize,
    breakatwhitespace=false,         
    breaklines=true,                 
    captionpos=b,                    
    keepspaces=true,                 
    numbers=left,                    
    numbersep=5pt,                  
    showspaces=false,                
    showstringspaces=false,
    showtabs=false,                  
    tabsize=2
}

%% file: sections_arxiv/abstract.tex
With the rapid growth in the number of Large Language Models (LLMs), there has been a recent interest in \emph{LLM routing}, or directing queries to the cheapest LLM that can deliver a suitable response. We conduct a minimax analysis of the routing problem, providing a lower bound and finding that a simple router that predicts both cost and accuracy for each question can be minimax optimal. Inspired by this, we introduce CARROT, a Cost AwaRe Rate Optimal rouTer that selects a model based on estimates of the models' cost and performance. Alongside CARROT, we also introduce the Smart Price-aware ROUTing (\newdata) dataset to facilitate routing on a wide spectrum of queries with the latest state-of-the-art LLMs. Using \newdata\ and prior benchmarks such as Routerbench and open-LLM-leaderboard-v2 we empirically validate CARROT's performance against several alternative routers. 
\blfootnote{$^\dagger$ primary contributors and corresponding authors (smrstep@umich.edu, maiapolo@umich.edu)\\ \phantom{fill f} * lead investigators (mikhail.yurochkin@ibm.com, smaity@uwaterloo.ca)}

%% file: sections_arxiv/intro_new.tex
\section{Introduction}
\label{sec:intro}



Large language models (LLMs) have demonstrated the capability to effectively address a diverse array of tasks across academic, industrial, and everyday settings \citep{minaee2024large}. This continued success has catalyzed the rapid development of new LLMs tailored for both general and specialized applications \citep{myrzakhan2024open}. While 
this offers practitioners increased flexibility, the vast number of available options may pose a daunting challenge in their real-world deployment. Particularly, determining the optimal LLM for a given query remains a significant challenge. In a perfect world, all queries can be routed to the most powerful model, but for many, this may quickly become prohibitively expensive.




A common approach to address this issue 
is \emph{routing} \citep{shnitzer2023large, hu2024routerbench, ong2024routellmlearningroutellms, notdiamond2023rorf,akota2024, chen2022efficientonlinemlapi, nguyen2025metallmhighperformantcostefficientdynamic}. There are two paradigms of routing; \emph{non-predictive} routers repeatedly call LLMs and evaluate the responses to select the best one for a given query. Examples include Fusion of Experts (FoE) \citep{wang2023fusing}, FrugalGPT \citep{chen2024frugalgptTMLR}, and techniques that cascade answers from weak to strong LLMs \citep{yue2024largelanguagemodelcascades}. The obvious disadvantage of non-predictive routing is the required inference of many LLMs for all queries, even those that are not suitable for the task at hand. As a workaround, researchers have also considered \emph{predictive routers}, which take LLM queries as inputs and output guesses at the most appropriate LLM.  
A key limitation of the prior literature on predictive routing is \emph{the avoidance of the cost prediction problem} for text generation in unknown queries. For example, \citet{shnitzer2023large} only considers performance prediction. In another direction, RouteLLM 
\citep{ong2024routellmlearningroutellms} and RoRF \citep{notdiamond2023rorf} take a step forward and implicitly incorporate model cost by creating binary routers that select between a large, costly model and a cheap, small model. However, they do not predict the cost of individual queries and, as we shall see, the reduced flexibility of binary routing leads to performance degradation in practice. A recent work, \citet{hu2024routerbench}, introduces a router that considers cost and accuracy, but they assume that cost is constant across all questions in the dataset. The works \citet{chen2022efficientonlinemlapi, nguyen2025metallmhighperformantcostefficientdynamic} consider dynamic vs. static cost prediction in routing, but their attention is limited to classification tasks; as the inference cost of LLMs is heavily dependent on the number of output tokens, studying this question for more open-ended prompts remains an important problem. Finally, \citet{akota2024} considers cost prediction, but their method does not generalize to unseen queries, undermining its use in more realistic applications. In summary, they require prior knowledge of the test set queries, as test-time routing decisions are made by solving a Linear Program for each query. In contrast, we will introduce a router that can handle an unknown stream of questions at test time.



To quantify the importance of cost prediction in routing, we provide a minimax analysis of the routing problem. In Theorem \ref{thm:lower-bound}, we establish a lower bound on minimax excess risk for any possible LLM router, in terms of the training sample size and certain quantities of the underlying prompt and model cost/accuracy distribution. Next, in Theorem \ref{thm:upper-bound} we show that a simple router based on predicting \emph{both LLM cost and accuracy} from a given prompt can achieve the minimax lower bound established in Theorem \ref{thm:lower-bound}. Collecting adequate data to train a router is challenging; for each LLM and every query one must collect a response \emph{and} an evaluation from a judge of that response. By necessity, this collection process must include inference from closed source models.  This makes the routing data gathering process expensive; the following informal Theorem emphasizes the importance of cost prediction in producing a router that makes the best use of this expensive data.

\begin{theorem}[Theorems \ref{thm:lower-bound} and \ref{thm:upper-bound} informal]
    An LLM router that predicts both cost and accuracy for every question and all models in a family can achieve optimal statistical efficiency.
\end{theorem}
Inspired by these findings (and to test them empirically), we introduce CARROT: a Cost AwaRe Rate Optimal rouTer and the Smart Price-aware ROUTing (\newdata) dataset. CARROT utilizes a simple two-stage approach. We first attain an estimator for each of the metrics (\eg\, cost and accuracy) for each model given a query, then we plug in these estimators into the formed risk function and select a model that minimizes the appropriate convex combination of the estimated metrics. 

The key learning step is attaining these aforementioned predictors, and this is where \newdata\ comes into play. \newdata\ covers 14 state-of-the-art language models (\eg, Llama-3-herd \citep{grattafiori2024llama3herdmodels}, GPT-4o \citep{openai2024gpt4}, \etc ) and approximately $45$k prompts from 6 benchmarks covering RAG, science, reasoning, and GPT-4 generated user queries. For all models, we use zero-shot prompting and corresponding chat templates to represent practical use cases and collect input and output token counts to allow flexibility when studying cost-performance trade-offs. The importance of a carefully curated dataset when studying routing cannot be overstated. In fact, on prior datasets such as  \citep{hu2024routerbench}, predictive routing offers almost no benefit over simply assigning prompts at random to the ``best" models. In contrast, on \newdata\ we will see that predictive routers are substantially better than this simple baseline. 

\subsection{Paper Outline} In Section \ref{sec:setup}, we introduce the routing problem and the ``plug-in" approach to routing that CARROT utilizes. In Section \ref{sec:lower-bound}, we provide our minimax analysis of the routing problem. This entails extending previous minimax studies in nonparametric classification \citep{audibert2007Fast} on two fronts: (1) with more than two classes, and (2)  with general losses beyond $0/1$-loss. Both of these extensions require us to introduce a generalized definition of margin (\cf\ eq. \eqref{eq:margin}), which reduces to the usual margin definition as in \citet{audibert2007Fast} when the classification task is binary and the loss is $0/1$. In Section \ref{sec:routing-application}, we introduce \newdata\ and empirically test the theoretical ideas discussed in the prior sections. To test CARROT's efficiency compared with prior routers, we utilize it to estimate the Pareto frontier of performance and cost trade-off on 
{RouterBench}
 \citep{hu2024routerbench}, 
open-LLM-leaderboard-v2\
\citep{open-llm-leaderboard-v2}, and our new \newdata
dataset. As a sneak peek, in Figure \ref{fig:IBMMix-spider}, we present the ratio of CARROT's performance to GPT-4o's \citep{openai2024gpt4} on several key benchmarks across diverse use cases represented in \newdata. At $30 \%$ of the cost, CARROT matches or exceeds the performance of GPT-4o on each benchmark. 

\subsection{Related literature}
\label{sec:related-work}

\paragraph{Performance vs cost trade-off in LLM predictions.} Several recent studies have explored optimizing the cost and performance trade-offs in the implementation of
large-language models (LLMs). LLM-BLENDER \citep{jiang2023llm} ensembles outcomes from
multiple LLMs to select the best response. Frugal-ML, Frugal-GPT \citep{chen2020frugalml,chen2024frugalgptTMLR} and FrugalFoE \cite{wang2023fusing} employ an LLM cascade to sequentially query LLMs until a reliable response is found. AutoMix \citep{madaan2023automix} relies on a smaller model to self-verify its response before potentially considering a larger model. While these approaches rely
on multiple LLM queries, our approach routes each query to a single LLM, an approach also considered in \citet{hu2024routerbench}. We complement these works by providing a statistically principled approach to learning this performance vs. cost trade-off. 

\begin{wrapfigure}[22]{r}{0.45\textwidth}
\vspace{-0.4cm}
  \begin{center}
     \includegraphics[width=\linewidth]{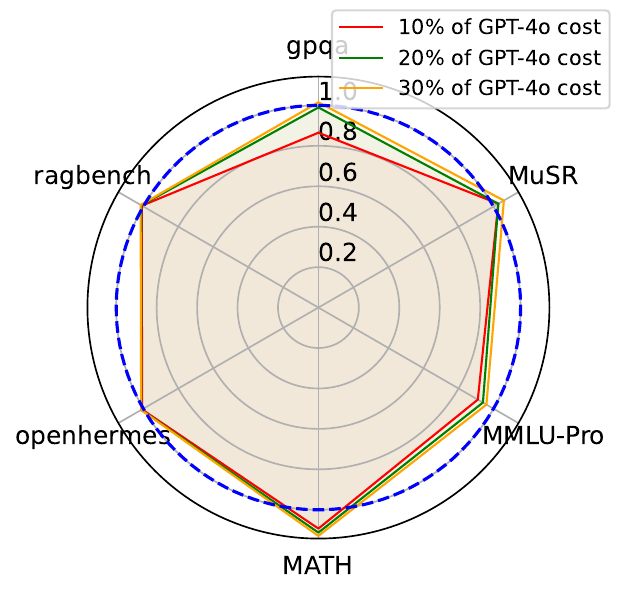} 
  \end{center}
  \caption{Percent of GPT-4o performance achieved by CARROT across datasets at various discounted costs, where the blue dotted line indicates similar ($100\%$) performance to GPT-4o.
  }
\label{fig:IBMMix-spider}
\end{wrapfigure}

\paragraph{Ensemble learning.} The routing problem is closely related to ensemble learning that combines multiple models to obtain better performance. Classical ensemble methods include bagging (bootstrap aggregating), boosting, and
stacking (model blending) \citep{breiman1996bagging,breiman1996stacked,freund1996experiments,friedman2001greedy,wolpert1992stacked}. Most of these works implicitly assume that the models in the ensemble have \emph{similar expertise}, and
thus it is beneficial to aggregate their predictions, whereas in our case, models may have \emph{complementary expertise}, and averaging their outputs might be detrimental because most of them may not be suitable for an input. Therefore, we choose to predict using the model with the best outcome, rather than aggregating them. 

\paragraph{Minimax studies in non-parametric classification.} One of the earliest works on the minimax rate of convergence in non-parametric classification is \citet{audibert2007Fast}. These techniques were later adopted for investigating the ability of transfer learning under a distribution shift
\citep{kpotufe2018Marginal,cai2019Transfer,maity2022minimax}. All of these works consider binary classification with $0/1$ loss. In comparison, our minimax investigation differs on two fronts: we extend the settings to classification with more than two classes and general cost functions. 


%% file: sections_arxiv/setup.tex
\section{Routing Problem and Plug-in Approach}
\label{sec:setup}

\subsection{Notation and preliminaries}

To begin, let us introduce our notation. We have $M$ pre-trained LLMs indexed as $m \in [M] = \{1, \dots, M\}$ and $K$ metrics indexed as $k\in [K] = \{1, \dots, K\}$. 
We denote a generic input or query as $X\in \cX$, where $\cX$ is the space of inputs.  Thus, for any input $X$, the metrics of interest are stored in a $M\times K$ matrix. We denote this matrix as $Y \in \reals^{M \times K}$, whose $(m, k)$-th entry $[Y]_{m, k}$ is the metric value for obtaining a prediction from the $m$-th model evaluated with respect to $k$-th metric.  For all metrics, we assume that a lower value is preferred. With this convention, we shall also refer to them as risks. For a probability distribution $P$ in the sample space $\cX \times \reals^{M\times K}$ we assume that the training dataset $\cD= \{(X_i, Y_i)\}_{i = 1}^n$ is an $\iid$ sample from $P$. 

 For the probability $P $ defined on the space $\cX \times \reals^{M \times K}$, we denote the marginal distribution of $X$ by $P_X$. Let us denote $\supp(\cdot)$ as the support of a probability distribution. Within the space $\reals^d$, we denote $\Lambda_d$ as the Lebesgue measure, $\|\cdot\|_2$ and $\|\cdot\|_\infty$ as the $\ell_2$ and $\ell_\infty$-norms, and $\cB(x, r, \ell_2)$ and $\cB(x, r, \ell_\infty)$ as closed balls of radius $r$ and centered at $x$ with respect to the $\ell_2$ and $\ell_\infty$-norms. 


\subsection{The Routing Problem}
We will consider a convex combination of our $K$ metrics with coefficients $\mu \in \Delta^{K-1} \triangleq \{(\mu_1, \dots, \mu_K): \mu_k \ge 0, \sum_k \mu_k = 1\}$ and a generic point $(X, Y)\sim P$. The $\mu$-th convex combination of the risks (or, $\mu$-th risk) can be written as $Y\mu\in \reals^{M}$, with the risk incurred for obtaining a prediction from the $m$-th model is 
\[
\textstyle [Y\mu]_m = \sum_{k = 1}^K [Y]_{m, k} \mu_k\,.
\] We want to learn a predictive router $g: \cX \to [M]$, that takes $X$ as an input and predicts the index of the LLM to be used for inference. The average $\mu$-th risk for using the router $g$ is 
\begin{equation}\label{eq:RM}
  \textstyle   \cR_P(g, \mu) = \Ex\big [\sum_{m = 1}^M[Y\mu]_m \bbI\{g(X) = m\} \big]\,. 
\end{equation} For a given $\mu$ let us refer to the minimizer $g_\mu^\star$ as an oracle router. The objective of the routing problem is to learn the oracle routers $g_\mu^\star$ at every value of $\mu$. 

\subsection{Plug-In Approach}
While one may minimize an empirical risk corresponding to $\cR_P(g, \mu)$ to estimate the oracle router at a particular $\mu$, this approach is not scalable, any small change in $\mu$ would require refitting a new router. Given this, we develop a plug-in approach which lets us estimate the oracle routers at every value of $\mu$. The key intuition lies within an explicit form of the $g_\mu^\star$ that we provide in the next lemma. 
\begin{lemma} \label{lemma:oracle-router}
    Let us define $\Phi(x) = \Ex [Y\mid X = x]$ and $\eta_{\mu, m}(x) = \sum_{k = 1}^K \mu_k [\Phi (x)]_{m, k} $. Then for any $\mu\in \Delta^{K-1}$ the oracle router that minimizes $\cR_P(g, \mu)$ is 
    \begin{align} \label{eq:oracle-router-2}
        \textstyle g_\mu^\star(X) & \textstyle = \argmin_m ~ \eta_{\mu, m} (X) \textstyle  = \argmin_m \big\{ \sum_{k = 1}^K \mu_k [\Phi (X)]_{m, k} \big\} \, \nonumber.
    \end{align}
\end{lemma}
The key conclusion of \ref{lemma:oracle-router} is the expression $ g_\mu^\star(X)= \argmin_m \{ \sum_{k = 1}^K \mu_k [\Phi (X)]_{m, k} \} $. 
It suggests a straightforward approach to estimate $g_\mu^\star(X)$ at all values of $\mu$.  Namely, we only need to plug-in an estimate of $\Phi(X) = \Ex[Y \mid X]$ to the expression of $g_\mu^\star(X)$. Compared to minimizing empirical risk at different values of $\mu$, this plug-in approach is more scalable if the practitioner plans on tuning $\mu$. 
 \begin{tcolorbox}[enhanced,title=CARROT,
        colframe=blue!40!black,
        colback=blue!2!white,
        fonttitle=\bfseries,
      attach boxed title to top text left={xshift=30mm,yshift=-2.5mm},
      boxed title
      style={size=small,colframe=blue!40!black,colback=blue!40!black}]
      \label{box:CARROT}
  CARROT is implemented in the following steps:
  \begin{enumerate}
      \item Learn an estimate $\widehat \Phi (X)$ of $\Phi(X)$ using a training split of a routing data set $\cD_{\text{tr}}$. 
      \item For a given convex combination of interest $\mu$, produce the router $\widehat g_\mu(X) = \argmin_m \widehat \eta_{\mu, m}(X)$ where $\widehat \eta_{\mu, m}(X) =  \sum_{k = 1}^K \mu_k [\widehat \Phi (X)]_{m, k}   $. 
  \end{enumerate}
\end{tcolorbox}

%% file: sections_arxiv/lower-bound.tex
\section{Statistical efficiency of CARROT} 
\label{sec:lower-bound}

In this section we establish that, under certain conditions, the plug-in approach to routing is minimax optimal. First we establish an information theoretic lower bound on the sample complexity for learning the oracle routers (\cf\ Theorem \ref{thm:lower-bound}). Next, we establish an upper bound for the minimax risk of plug-in routers (\cf\ Theorem \ref{thm:upper-bound}). Finally, we show that under sufficient conditions on the estimates of $\Ex[Y\mid X]$ the sample complexity in the upper bound matches the lower bound. We will also generalize slightly to the setting where the last $K_2$ metrics are known functions of $X$, \ie\ for $m \in [M], k \in \{K - K_2 +1 , \dots K\}$ there exist known functions $f_{m, k}: \cX \to \reals$ such that $[Y]_{m, k} = f_{m, k}(X)$. Since $\Ex[[Y]_{m, k}\mid X] = f_{m, k}(X)$ are known for $k \ge K - K_2 +1 $ they don't need to be estimated. 

\subsection{Technical Assumptions}



The technical assumptions of our minimax study are closely related to those in investigations of non-parametric binary classification problems with $0/1$ loss functions, \eg\  \citet{cai2019Transfer,kpotufe2018Marginal,maity2022minimax,audibert2007Fast}. In fact, our setting generalizes the classification settings considered in these papers on multiple fronts: (i) we allow for general loss functions, (ii) we allow for more than two classes, and (iii) we allow for multiple objectives. 

To clarify this, we discuss how binary classification is a special case of our routing problem. 

\begin{example}[Binary classification with $0/1$-loss] \label{example:binary-classification}
    Consider a binary classification setting with $0/1$-loss: we have the pairs $(X, Z) \in \cX \times \{0, 1\}$ and we want to learn a classifier $h: \cX \to\{0, 1\} $ to predict $Z$ using $X$. This is a special case of our setting with $M = 2$ and $K= 1$, where for $m \in \{0, 1\}$ the $[Y]_{m, 1} = \bbI\{Z \neq m\}$. Then the risk for the classifier $h$, which can also be thought of as a router, is 
\begin{align*}
\textstyle \cR_P(h) & \textstyle = \Ex\big[\sum_{m \in \{0, 1\}}[Y]_{m, 1} \bbI\{h(X) = m\} \big] = \Ex\big[ \bbI\{h(X) \neq Z\} \big]\,,
\end{align*} the standard misclassification risk for binary classification. 
\end{example}


We assume that $\supp(P_X)$ is a compact set in $\reals^d$. This is a standard assumption in minimax investigations for non-parametric classification problems \citep{audibert2007Fast,cai2019Transfer,kpotufe2018Marginal,maity2022minimax}. 
Next,  we place H\"older smoothness conditions on the functions $\Phi_m^\star$. This controls the difficulty of their estimation. For a tuple $s = (s_1 , \dots, s_d) \in (\bN \cup \{0\})^d$ of $d$ non-negative integers  define $|s| = \sum_{j = 1}^d s_j$ and for a function $\phi: \reals^d\to \reals$ and $x = (x_1, \dots, x_d) \in \reals^d$ define the differential operator, assuming that such a derivative exists: 
\begin{equation}
  \textstyle  D_s(\phi, x) = \frac{\partial^{|s|}\phi(x)}{\partial x_1^{s_1} \dots \partial x_d^{s_d}}\,. 
\end{equation} Using this differential operator we now define the H\"older smoothness condition: 

\begin{definition}[H\"older smoothness]
   For $\beta, K_\beta >0$ we say that $\phi:\reals^d \to \reals$ is $(\beta, K_\beta)$-H\"older smooth on a set $ A \subset \reals^d$ if it is $\lfloor \beta \rfloor$-times continuously differentiable on $A$ and for any $x, y \in A $ 
   \begin{equation}
       |\phi(y) - \phi_x ^{(\lfloor \beta \rfloor)}(y)| \le K_\beta \|x - y\|_2^\beta\,,
   \end{equation} where 
$\phi_x ^{(\lfloor \beta \rfloor)}(y) = \sum_{|s| \le \lfloor \beta \rfloor} D_s(\phi, x) \{\prod_{j = 1}^d(y_j - x_j)^{s_j}\} $ is the $\lfloor \beta \rfloor$-order Taylor polynomial approximation of $\phi(y)$ around $x$. 
\end{definition}
With this definition, we assume the following:
\begin{assumption}\label{assmp:smooth}
    For $m \in [M]$ and $k \in [K_1]$ the 
    $[\Phi(X)]_{m, k}$ is $(\gamma_{k}, K_{\gamma, k})$-H\"older smooth. 
\end{assumption} 
This smoothness parameter will appear in the sample complexity of our  plug-in router. Since the $[\Phi(X)]_{m, k}$ are known for $k \ge K_1 + 1$ we do not require any smoothness assumptions on them.

Next, we introduce the \emph{margin condition}, which quantifies the difficulty in learning the oracle router.  For a given $\mu$ define the margin as the difference between the minimum and second minimum of the risk values: 
{ \begin{equation}\label{eq:margin}
    \begin{aligned}
        & \textstyle \Delta_\mu(x) =  
    \begin{cases}
       \min\limits_{m \notin g_\mu(x)} \eta_{\mu, m}(x) - \min\limits_m \eta_{\mu, m}(x) & \text{if} ~ g_\mu^\star(x) \neq [M]\\ 
       0 & \text{otherwise}.
       \end{cases} 
    \end{aligned}
\end{equation}}

The margin determines the difficulty in learning the oracle router. A query $X$ with a small margin gap is difficult to route, because to have the same prediction as the oracle, \ie\  $\argmin_{m} \hat \eta_{\mu, m}(X) = \argmin_{m} \eta_{\mu, m}^\star(X)$ we need to estimate $ \eta_{\mu, m}^\star(X)$ with high precision. In the following assumption, we control the probability of drawing these ``difficult to route'' queries.

\begin{assumption}[Margin condition]\label{assmp:margin}
    For $\alpha, K_\alpha >0$ and any $t > 0$ the margin $\Delta_{\mu}$ \eqref{eq:margin} satisfies: \begin{equation}
        P_X \big\{0 < \Delta _\mu(X) \le t\big \}  \le K_\alpha t^{\alpha}\,. 
    \end{equation}
\end{assumption}
Following \citet{audibert2007Fast}, we focus on the cases where $\alpha < d$ and for every $k$ the $\alpha \gamma_k < d$. This helps to avoid trivial cases where routing decisions are constant over $P_X$ for some $\mu$.  
Next, we assume that $P_X$ has a density $p_X$ that satisfies a strong density condition described below.
\begin{assumption}[Strong density condition] \label{assmp:strong-density}
Fix constants $c_0, r_0> 0$ and $0 \le \mu_{\min}  \le \mu_{\max} < \infty$. We say $P_X$ satisfies the strong density condition if its support is a compact $(c_0, r_0)$-regular set and it has density $p_X$ which is bounded: $\mu_{\min} \le p_X (x) \le \mu_{\max} $ for all $x$ within $\supp(P_X)$. A set $A \subset \reals^d$ is $(c_0, r_0)$-regular if it is Lebesgue measurable and 
$\text{for any} ~ 0 < r \le r_0, ~  x \in A$ it satisfies
\begin{equation}
    \Lambda_d (A \cap \cB(x, r, \ell_2)) \ge c_0 \Lambda_d(\cB(x, r, \ell_2)). 
\end{equation} 
\end{assumption}
This is another standard assumption required for minimax rate studies in nonparametric classification problems \citep{audibert2007Fast,cai2019Transfer}. All together, we define 
$\cP$, as the class of probabilities $P$ defined on the space $\cX \times \cY$ for which $P_X$  is  compactly supported and satisfies the strong density assumption \ref{assmp:strong-density} with parameters $(c_0, r_0, \mu_{\min}, \mu_{\max})$, and the H\"older smoothness assumption \ref{assmp:smooth} and the $(\alpha, K_\alpha)$-margin condition in Assumption \ref{assmp:margin} hold. We shall establish our minimax rate of convergence within this probability class.

    

\subsection{The lower bound} 
Rather than the actual risk $\cR_P(\mu, g)$, we establish a lower bound on the excess risk:
\begin{equation}\label{eq:excess-risk}
    \cE_P(\mu, g) = \cR_P(\mu, g) - \cR_P(\mu, g_\mu^\star)\,,
\end{equation} that compares the risk of a proposed router to the oracle one. We denote $\Gamma = \{g: \cX \to [M]\}$ as the class of all routers. For an $n \in \bN$ we refer to the map $A_n: \cZ^n \to \Gamma$, which takes the dataset $\cD_n $ as an input and produces a router $A_n(\cD_n): \cX \to [M]$, as an algorithm. Finally, call the class of all algorithms that operate on $\cD_n$ as $\cA_n$. The following Theorem describes a lower bound on the minimax risk for any such algorithm $A_n$. 
\begin{theorem}\label{thm:lower-bound}
    For an $n \ge 1$  and $A_n \in \cA_n$ define  $\cE_P(\mu, A_n) = \Ex_{\cD_n}\big[\cE_P\big(\mu, A_n(\cD_n)\big)\big]$ as the excess risk of an algorithm $A_n$. There exists a constant $c> 0$ that is independent of both $n$ and $\mu$ such that for any $n\ge 1$ and $\mu\in \Delta^{K-1}$ we have the lower bound
    \begin{equation}\label{eq:lower-bound}
      \textstyle  \min\limits_{A_n \in \cA_n} \max\limits_{P \in \cP} ~~ \cE_P(\mu, A_n) \ge c \big \{\sum_{k = 1}^{K_1} \mu_k n^{- \frac{\gamma_k}{2\gamma_k + d}}\big\}^{1+\alpha} \,.
    \end{equation}
\end{theorem} 
This result is a generalization of that in \citet{audibert2007Fast}, which considers binary classification. 
\begin{remark} \label{remark:minimax-lower-bound}
    Consider the binary classification in Example \ref{example:binary-classification}. Since $K = 1$, the lower bound simplifies to $\cO(n^{-\nicefrac{\gamma_1 (1+ \alpha)}{2\gamma_1 + d}})$,  which matches with the rate in \citet[Theorem 3.5]{audibert2007Fast}. 
    Beyond $0/1$ loss, our lower bound also establishes that the rate remains identical for other classification loss functions as well. 

\end{remark}

\subsection{The upper bound }\label{sec:upper-bound}
Next, we show that if the algorithm $A_n$ corresponds to CARROT, the performance of $\hat{g}_{\mu}$ matches the lower bound in Theorem \ref{thm:lower-bound} (\cf\ equation \ref{eq:lower-bound}). En-route to attaining $\hat{g}_{\mu}$, we need an estimate $\widehat \Phi(X)$ of $\Phi(X) = \Ex_P[Y \mid X ]$. 
We begin with an assumption for a  rate of convergence for $[\widehat \Phi(X)]_{m, k}$. 
\begin{assumption} \label{assmp:convergence}
    For some constants $\rho_1, \rho_2 > 0$ and any $n \ge 1$ and $t > 0$ and almost all $X$ with respect to the distribution $P_X$ we have the following concentration bound:
    \begin{align}\label{eq:concentration-phi}
        \max_{P\in \cP} P \big \{ \max_{m, k} a_{k, n}^{-1}\big |[\widehat \Phi (X)]_{m, k} - [\Phi  (X)]_{m, k}\big |
        \ge t\big \}  
        \le  \rho_1 \exp\big (- \rho_2  t^2 \big )\,,  
    \end{align}  where for each $k$ the  $\{a_{k,n}; n \ge 1\}\subset (0, \infty)$ is a sequence that decreases to zero. 
\end{assumption}
Using this high-level assumption, in the next theorem, we establish an upper bound on the minimax excess risk for CARROT that depends on both $a_{k, n}$ and $\mu$.  
\begin{theorem}[Upper bound]\label{thm:upper-bound}
  Assume \ref{assmp:convergence}.   If all the $P\in \cP$ satisfy the margin condition \ref{assmp:margin} with the parameters $(\alpha, K_\alpha)$ then there exists a $K> 0$ such that for any $n \ge 1$ and $\mu\in \Delta^{K-1} $ the excess risk for the router $\widehat g_\mu$ in Algorithm \ref{box:CARROT} is upper bounded as 
    \begin{equation}
        \max_{P\in \cP} \Ex_{\cD_n}\big [\cE_P(\widehat g_\lambda,\lambda)\big ] \le\textstyle K \big \{\sum_{k = 1}^{K_1} \mu_k a_{k, n}\big\}^{1+\alpha} \,.
    \end{equation}
\end{theorem} 

\begin{remark}[Rate efficient routers] \label{cor:efficient-routers}
    When $a_{k, n} = n^{-\nicefrac{\gamma_k}{(2\gamma_k +d)}}$ the upper bound in Theorem \ref{thm:upper-bound} has the $\cO(\{\sum_{k = 1}^{K_1}\mu_k n^{-\nicefrac{\gamma_k}{(2\gamma_k+d)}}\}^{1+\alpha})$-rate, which is identical to the rate in the lower bound (\cf\ Theorem \ref{thm:lower-bound}), suggesting that the minimax optimal rate of convergence for the routing problem is 
\begin{equation}
    \label{eq:minimax-rate}
     \textstyle  \min\limits_{A_n \in \cA_n} \max\limits_{P \in \cP} ~~ \cE_P(A_n, \lambda) \asymp \textstyle  \cO\big ( \big \{\sum_{k = 1}^{K_1} \mu_k n^{- \frac{\gamma_k}{2\gamma_k + d}}\big\}^{1+\alpha}\big ) \,.
\end{equation}
   Following this, we conclude: When $a_{k, n} = n^{-\nicefrac{\gamma_k}{(2\gamma_k +d)}}$ the plug-in approach in Algorithm \ref{box:CARROT}, in addition to being computationally efficient, is also minimax rate optimal. 
\end{remark} 
An example of an estimator $\widehat{\Phi}$ that meets the needed conditions for $a_{k, n} = n^{-\nicefrac{\gamma_k}{(2\gamma_k +d)}}$ to hold is described in Appendix \ref{sec:reg-fn-estimate}. 

%% file: sections_arxiv/experiment.tex
\section{Routing in benchmark case-studies}
\label{sec:routing-application}
We use CARROT (Algorithm \ref{box:CARROT}) to perform routing on several benchmark datasets. 

\subsection{Datasets} \label{sec:datasets}



\paragraph{RouterBench:} RouterBench \citep{hu2024routerbench} is a benchmark dataset for routing tasks consisting of approximately 30k prompts and responses from eleven ($M = 11$) different LLMs.  The data includes prompts from 8 benchmarks covering commonsense reasoning, knowledge-based understanding, conversation, math, and coding. 
\paragraph{Open LLM leaderboard:} The {Open LLM leaderboard v2}\footnote{\url{https://huggingface.co/spaces/open-llm-leaderboard/open\_llm\_leaderboard}} \citep{open-llm-leaderboard-v2} is an open-source benchmarking platform that comprises responses and evaluations of a collection of LLMs on six benchmarks comprising a diverse collection of tasks. 
\paragraph{\newdata:} We introduce (and evaluate CARROT on) \newdata, a large and diverse dataset designed for training and evaluating routers. \newdata\ integrates $M=15$ state-of-the-art language models (see Table \ref{tab:price_by_token_newdata}) and prompts from 6 benchmarks, including GPQA \citep{rein2023gpqagraduatelevelgoogleproofqa}, MuSR \citep{sprague2024musrtestinglimitschainofthought}, MMLU-Pro \citep{wang2024mmluprorobustchallengingmultitask}, MATH \citep{hendrycks2021measuringmathematicalproblemsolving}, OpenHermes \citep{teknium_openhermes_2.5}, and RAGBench \citep{friel2025ragbenchexplainablebenchmarkretrievalaugmented}. Compared to existing routing benchmarks such as RouterBench, \newdata\ offers several key advantages: 
\begin{enumerate}
    \item \newdata\ encompasses a highly diverse set of questions, including instruction queries.
    \item Unlike previous benchmarks, it does not rely on few-shot prompting and utilizes chat templates appropriate for each model, making it more representative of real-world use cases.
    \item It leverages LLaMa-3.1-70b-Instruct \citep{grattafiori2024llama3herdmodels} to evaluate LLM responses against the ground truth, similarly to \citet{ni2024mixeval}. This is crucial for evaluating on open-ended instruction queries as well as mitigating errors associated with traditional automatic evaluation methods like exact match.
    \item We provide input and output token counts for each LLM-prompt pair, enabling flexibility when conducting cost-aware analysis.
\end{enumerate}
We have released the \newdata\ in {huggingface }
 and will open-source a platform that allows practitioners to extend \newdata\ by adding new queries and seamlessly evaluating state-of-the-art models on them. For further details, please refer to Appendix \ref{sec:append:sprout}.


\subsection{Plug-in Estimates}
\label{sec: exp-estimators}
CARROT requires an estimate for the function $\Phi^\star_m(X) = \Ex_P[ Y_m\mid X]$. In our benchmark tasks, $Y_m$ is 2-dimensional, consisting of model performance measured as accuracy and model cost measured in dollars. In all routing datasets, $Y_{\text{acc}, m}$ is binary, and thus 
we can view its estimation as a binary classification problem, where our objective is to predict the probability that $m$-th model will answer the question $X$ correctly, \ie\ $P_m(X) = P(Y_{\text{acc}, m} = 1\mid X)$) $Y_{\text{acc}, m}$.  We train several multi-label classification models $\widehat P: \cX \to [0, 1]^M$ on a training data split consisting of $80\%$ of the full dataset, where the $m$-th coordinate of $\widehat{P}(X)$ is the predicted probability that $m$-th model accurately answers the question $X$. In the RouterBench and \newdata\ task the cost must also be estimated. We train multi-label regression models $\widehat C: \cX \to \reals^M$, where $\hat{C}_m(X) =\mathbb{E}[Y_{\text{cost}, m}|X]$ is the estimated cost of calling model $m$ for query $X$. To train the cost or performance predictors we consider two procedures:
\begin{enumerate}
    \item  {\bf CARROT (KNN):} We embed the model inputs using the {text-embedding-3-small} model from OpenAI \citep{openai2023textembedding3small}. On these text embeddings, we train a multi-label K-nearest-neighbors (KNN) classifier/regressor.

    \item {\bf CARROT (Roberta):} We fine-tune the pre-trained weights of the {{roberta-base}}\footnote{\url{https://huggingface.co/FacebookAI/roberta-base}} architecture. In order to enhance efficiency, across $m$ we allow $\widehat{P}$ to share the same network parameters, except for the final classification/regression layer. 
\end{enumerate}

\begin{figure}
       \centering
       \begin{subfigure}[b]{0.45\textwidth}
           \centering
           \includegraphics[width=0.8\linewidth]{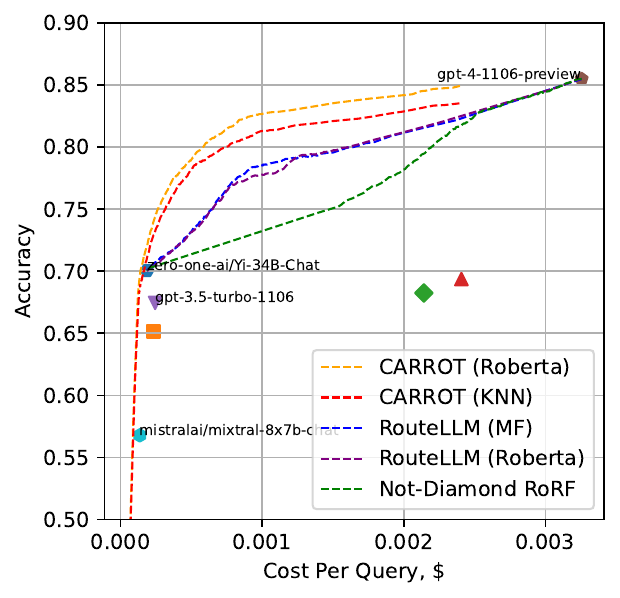} 
           \caption{CARROT vs. Binary Routers on Routerbench.}
           \label{fig:RouterBench-main}
       \end{subfigure}
       \hspace{0.05\textwidth} 
       \begin{subfigure}[b]{0.45\textwidth}
           \centering
           \includegraphics[width=0.8\linewidth]{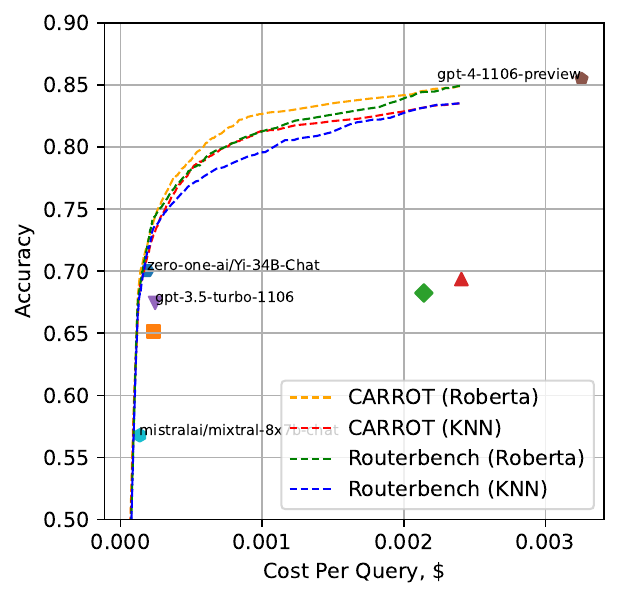}
           \caption{CARROT vs. Routerbench on Routerbench.}
           \label{fig:RouterBench-Routerbench}
       \end{subfigure}
       \caption{Performance of several routers and individual LLMs on test data-split in Routerbench.}
       \label{fig:mainfig}
   \end{figure}

\subsection{Baseline methods}
\emph{Zero Router:} The zero router is a simple check to see if predictive routing is of any value on a given data set. This benchmark randomly assigns prompts to the best performing (in terms of cost/accuracy trade-off) in the dataset.

\emph{Binary Routers:} \citet{ong2024routellmlearningroutellms} (RouteLLM) proposes a collection of methods for learning binary routers from preference data (data consisting of queries $q$ and labels $l_{i,j}$ indicating a winner between model $i$ and $j$). While the usage of preference data is slightly different from ours, we implement their methods on our data by creating pseudo-preference data between two models. In particular, we select a costly and high-performing model and a cheaper model and say the costly model wins if and only if it is correct while the cheaper model is incorrect. On this pseudo preference data, we fit two methods from \citet{ong2024routellmlearningroutellms} for learning win probabilities between expensive and cheap models: the first is a matrix factorization method, called {\bf RouteLLM (MF)}, while the second uses fine-tuned {roberta-base}, called {\bf RouteLLM (Roberta)}. A follow-up method to these is Routing on Random Forests (RoRF) from Not-Diamond \citep{notdiamond2023rorf}, referred to as {\bf Not-Diamond RoRF}. This method uses a text-embedder and random forest model to predict the win probability; we provide a comparison to this method with the {text-embedding-3-small} embedder from OpenAI. 

\emph{Cost-Unaware Routers:} Another class of routers routes to multiple models but does not attempt to predict inference cost for each question \citep{chen2022efficientonlinemlapi, nguyen2025metallmhighperformantcostefficientdynamic, hu2024routerbench}. The most comparable router to ours is the \textbf{Routerbench} router proposed in \citep{hu2024routerbench}. Their router is essentially CARROT with a constant cost predictor: For each test question, the predicted model use cost is the average cost over the training set for that model, while the performance predictor matches what we describe in Subsection \ref{sec: exp-estimators} 

\subsection{Results} 



\paragraph{Performance against baselines:} 
In Figures \ref{fig:RouterBench-main} and \ref{fig:rexp_open_llm}, we see that CARROT handily beats routers that only consider two models. This is due to the fact that we route to \emph{all possible models}, which increases the accuracy coverage and decreases the cost of the cheapest accurate model for a given query. In Figures \ref{fig:RouterBench-Routerbench} and \ref{fig:ROOT In Distribution}, we see that CARROT offers marginal improvements over the Routerbench router \citep{hu2024routerbench} that does not attempt to predict cost. Together, these findings suggest that while good cost prediction is important for achieving the most efficient price/accuracy trade-off, the large majority of routing performance is made up of correctly assessing the accuracy of each model in the family.

\paragraph{CARROT can (sometimes) out-perform the best model:} In {RouterBench} we were unable to achieve significantly better accuracy than GPT-4; however, we were able to greatly reduce the prediction cost. Likewise on \newdata\ we are able to process the test set at a fraction of the cost of o3-mini, but at its best CARROT cannot exceed the o3-mini's accuracy. On the other hand, we showed that CARROT can outperform the best model (Qwen2-72B) by a large margin in {Open LLM leaderboard v2} (see Figure \ref{fig:rexp_open_llm}). The difference is likely due to the existence of a singular top-performing model, or multiple models with comparable best accuracies.



\begin{figure}[ht]
       \centering
       \begin{subfigure}[b]{0.47\textwidth}
           \centering
           \includegraphics[width=\linewidth]{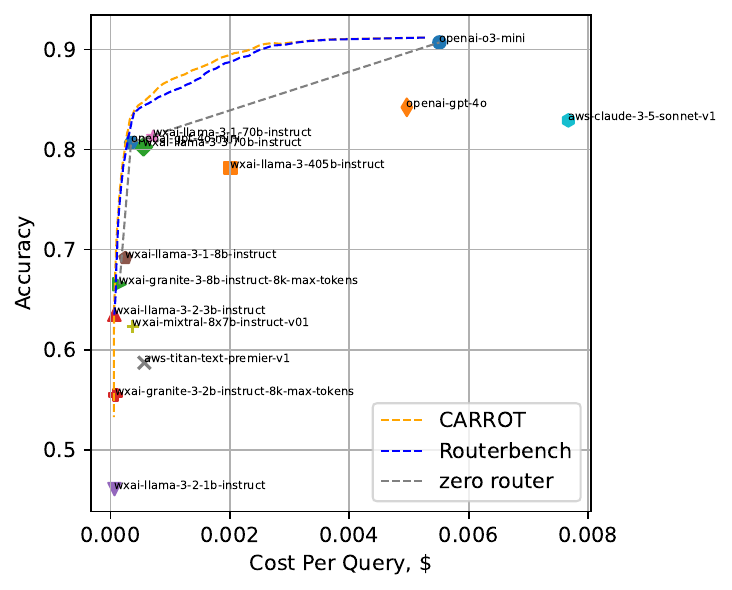}
           \caption{CARROT vs. Routerbench and zero router on \newdata.}
           \label{fig:ROOT In Distribution}
       \end{subfigure}
       \hspace{0.04\textwidth} 
       \begin{subfigure}[b]{0.47\textwidth}
           \centering
           \includegraphics[width=0.86\linewidth]{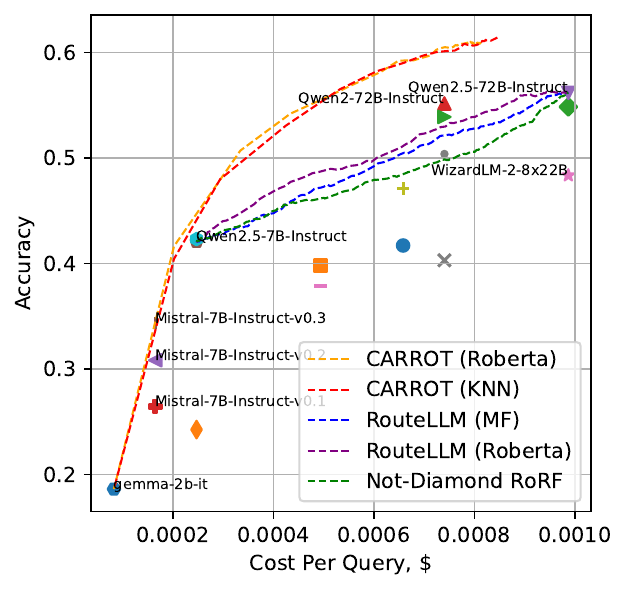}
    \caption{CARROT vs. Binary Routers on Open-LLM-Leaderboard v2.}
           \label{fig:rexp_open_llm}
       \end{subfigure}
       \caption{CARROT routing analysis on the \newdata\ and Open-LLM-Leaderboard-v2 dataset.}
       \label{fig:fmselect}
   \end{figure}




\paragraph{Predictive routing is highly valuable for SPROUT:} \citet{hu2024routerbench} conclude that "none of the routing algorithms significantly outperform the baseline zero router" on the Routerbench data set, where the zero router linearly interpolates between models on the frontier of cost and accuracy. This suggests that on Routerbench, predictive routing does not provide any significant benefit. In our view, this is not a shortcoming of predictive routing; rather, it is a shortcoming of the Routerbench dataset itself. Indeed, in Figure \ref{fig:ROOT In Distribution} we show that for a carefully collected dataset, predictive routing can provide substantial gains. On \newdata\ both CARROT and the cost-unaware router provide substantial improvement over the zero-router. This suggests that \newdata\ is an important introduction to the landscape of routing data sets.


%% file: sections_arxiv/conclusion.tex
\section{Discussion}
\label{sec:discussion}

We introduced CARROT, a plug-in based router that is both computationally and statistically efficient. The computational efficiency stems from the requirement of merely calculating the plug-in estimators (see Algorithm \ref{box:CARROT}) to  perform routing 
Since collecting adequate data for router training might be challenging, we investigate CARROT's statistical efficiency in routing through a minimax rate study. To establish the statistical efficiency of CARROT, we have provided an information-theoretic lower bound on the excess risk of any router in Theorem \ref{thm:lower-bound} and corresponding upper bound for CARROT in Theorem \ref{thm:upper-bound}. To ensure a broad scope for CARROT to a diverse set of queries \emph{and} the latest state-of-the-art LLMs, we also introduced the \newdata\ dataset. 

Our routing and data approach is designed to be forward-looking. CARROT can incorporate many metrics besides performance and cost; an important next step is to explore which other metrics can improve LLM-based decision-making in practice. A related future goal is to benchmark our SPROUT-trained router on enterprise use cases like the Domain Intelligence Benchmark Suite (DIBS)\footnote{\url{https://www.databricks.com/blog/benchmarking-domain-intelligence}} to locate areas of improvement needed in our data.


%% file: sections_arxiv/supp.tex
\section{\newdata\ Construction details and plots}\label{sec:append:sprout}
In this section, we discuss data details for \newdata. \newdata\ will be released on HuggingFace hub as a HuggingFace datasets object. For convenience, the data is pre-divided into train, validation, and test splits. Consider the training set as an example; the features of this split are



\begin{lstlisting}[language=Python]

features = ['key', 'dataset', 'dataset level', 'dataset idx', 'prompt', 
            'golden answer', 'o3-mini', 'aws-claude-3-5-sonnet-v1', 
            'titan-text-premier-v1', 'openai-gpt-4o', 
            'openai-gpt-4o-mini', 'granite-3-2b-instruct', 
            'granite-3-8b', 'llama-3-1-70b-instruct', 
            'llama-3-1-8b-instruct', 'llama-3-2-1b-instruct', 
            'llama-3-2-3b-instruct', 'llama-3-3-70b-instruct', 
            'llama-3-405b-instruct', 'mixtral-8x7b-instruct-v01']

\end{lstlisting}

Each key corresponds to another list. \texttt{"prompt"} contains the model queries, the \texttt{"dataset"} list indicates which sub-task a given query falls in (\cf\ Table \ref{tab:dataset_splits} for info), and \texttt{golden answer} contains a desirable response for each query. Finally, the model keys each correspond to a list of dictionaries that contains further information on the responses of that model. The important keys in each dictionary of the list are \texttt{["num input tokens", "num output tokens", "response", "score"]}. They contain the number of input tokens for a query, the number of output tokens a model gives in response to a query, the actual response of the model, and finally the score that the judge provides for the response (using the corresponding \texttt{golden answer} entry). The conversion of token count to cost is given in Table \ref{tab:price_by_token_newdata} and additional details on the judging process are described in Section \ref{sec: sprout collection}.

\subsection{\newdata\ ingredients}
Table \ref{tab:dataset_splits} gives the benchmark ingredients for \newdata. Namely, we use the
MATH Lvl 1-5 \citep{hendrycks2021measuringmathematicalproblemsolving}, MMLU-PRO \citep{wang2024mmluprorobustchallengingmultitask}, 
GPQA \citep{rein2023gpqagraduatelevelgoogleproofqa}, MUSR \citep{sprague2023musr}, RAGBench \citep{friel2025ragbenchexplainablebenchmarkretrievalaugmented}, and openhermes \citep{teknium_openhermes_2.5} datasets. These six benchmarks are varied and designed to simulate real-world scenarios where LLMs encounter a wide range of prompts. MATH focuses solely on mathematical word problems, whereas MMLU-PRO and GPQA include both mathematical and advanced science questions. 
MuSR serves as a benchmark for assessing multistep soft reasoning tasks framed within natural language narratives. RAGBench is a retrieval augmented generation (RAG) benchmark dataset collected from Question-Answer (QA) datasets (CovidQA \citep{moller-etal-2020-covidqa}, PubmedQA \citep{jin-etal-2019-pubmedqa}, HotpotQA \citep{yang2018hotpotqa}, MS Marco \citep{nguyen2016msmacro}, CUAD \citep{hendrycks2021cuad}, EManual \citep{nandy-etal-2021-emanual}, TechQA \citep{castelli-etal-2020-techqa}, FinQA \citep{chen-etal-2021-finqa}, TAT-QA \citep{zhu-etal-2021-tatqa}, ExpertQA \citep{malaviya2024expertqa}, HAGRID \citep{kamalloo2023hagrid}), as well as one that was specifically adapted for RAG (DelucionQA \citep{sadat2024delucionqa}). This measures the ability of a LLM to incorporate retrieved documents along with user queries to generate accurate answers for problems that require in-depth domain knowledge. As such, RAGbench is grouped by the needed domain knowledge: bio-medical research (PubmedQA,
CovidQA), general knowledge (HotpotQA, MS Marco, HAGRID, ExperQA), legal contracts (CuAD),
customer support (DelucionQA, EManual, TechQA), and finance (FinBench, TAT-QA). Finally, openhermes is a collection of GPT4 generated questions designed to emulate real user queries to an LLM. 




\begin{table}
    \centering
        \caption{Dataset Splits for \newdata.}
    \begin{tabular}{lccc}
        \toprule
        \textbf{Benchmark} & \textbf{Train} & \textbf{Validation} & \textbf{Test} \\
        \midrule
        ragbench/expertqa & 98 & 17 & 16 \\
        MATH (test) & 1725 & 363 & 384 \\
        ragbench (emanual) & 82 & 27 & 23 \\
        ragbench (cuad) & 151 & 35 & 29 \\
        MuSR & 178 & 35 & 35 \\
        MATH & 5217 & 1061 & 1134 \\
        MuSR (team allocation) & 157 & 52 & 41 \\
        ragbench (hagrid) & 92 & 23 & 17 \\
        gpqa (extended) & 368 & 89 & 84 \\
        MuSR (object placements) & 169 & 47 & 34 \\
        ragbench (pubmedqa) & 92 & 14 & 26 \\
        ragbench (hotpotqa) & 89 & 22 & 21 \\
        ragbench (msmarco) & 85 & 24 & 23 \\
        ragbench (techqa) & 85 & 24 & 23 \\
        MMLU-Pro & 8204 & 1784 & 1798 \\
        openhermes & 13703 & 2917 & 2835 \\
        ragbench (tatqa) & 90 & 17 & 25 \\
        ragbench (finqa) & 97 & 15 & 20 \\
        ragbench (covidqa) & 162 & 38 & 41 \\
        ragbench (delucionqa) & 124 & 32 & 28 \\
        TOTAL & 30968 & 6636 & 6637 \\
        \bottomrule
    \end{tabular}

    \label{tab:dataset_splits}
\end{table}

\subsection{\newdata\ models and response collection}
\label{sec: sprout collection}
Table \ref{tab:price_by_token_newdata} provides the models and their associated costs that a router trained on \newdata\ can select between. The input and output token counts are collected by simply gathering the count of the tokenized queries and outputs of a model from its tokenizer. In order to emulate real-world use cases, responses from each LLM are collected using a corresponding \emph{chat template} with a generic prompt and \emph{zero shot prompting}.

Given the use of chat templates and zero-shot prompting, evaluation is challenging because model responses will not necessarily follow a specific format. To alleviate this, we adopt the evaluation protocol from MixEval \citep{ni2024mixeval} and use LLama-3.1-70B as a grader to score model queries against a given gold standard answer. The prompt format that we use is provided in \ref{sec: JP}. Note that this prompt format needs to be converted to openai-api compatible messages while prompting the LLMs, which can be inferred from the special delimiters contained within the prompt format.

\begin{table}
    \centering
    \caption{Models in \newdata\ dataset and their API prices according to token counts.}
    \begin{tabular}{c c c } 
 \toprule
 \multirow{2}{*}{Model} & Input Token Cost & Output Token Cost\\
 & (in \$ per 1M tokens) & (in \$ per 1M tokens)\\
 \hline\hline
 openai-o3-mini & 1.1 & 4.4 \\ 
 \hline
 claude-3-5-sonnet-v1 & 3 & 15  \\ 
 \hline
 titan-text-premier-v1 & 0.5 & 1.5  \\
 \hline
 openai-gpt-4o & 2.5 & 10  \\
 \hline
 openai-gpt-4o-mini & 0.15 & 0.6  \\
 \hline
 openai-o3-mini & 1.1 & 4.4  \\
 \hline
 granite-3-2b-instruct & 0.1 & 0.1  \\ 
 \hline
 granite-3-8b-instruct & 0.2 & 0.2  \\
 \hline
 llama-3-1-70b-instruct & 0.9 & 0.9  \\ 
 \hline
 llama-3-1-8b-instruct & 0.2 & 0.2  \\ 
 \hline
  llama-3-2-1b-instruct & 0.06 & 0.06  \\ 
 \hline
 llama-3-2-3b-instruct & 0.06 & 0.06  \\ 
 \hline
 llama-3-3-70b-instruct & 0.9 & 0.9  \\ 
 \hline
 mixtral-8x7b-instruct & 0.6 & 0.6  \\ 
 \hline
 llama-3-405b-instruct & 3.5 & 3.5  \\ 
 \bottomrule
\end{tabular}
    \label{tab:price_by_token_newdata}
\end{table}

\section{Additional Plots and Experimental details}
\subsection{RouteBench}
Figure \ref{fig:routerbench-table} lays out the models and benchmarks present in the Routerbench dataset. To implement the transformer-based plug-in estimate of cost and accuracy, we utilize the \texttt{roberta-base} architecture with a learning rate of \texttt{3e-5} and a weight decay of $0.01$. A training, validation, test split of $0.72$, $0.8$, $0.2$ is used. Learning proceeds for $5$ epochs, and the model with the best validation performance is saved at the end. To fit the KNN-based router, the OpenAI text-embedding-small-3 model is used, while the KNN regressor utilizes the $40$-nearest neighbors measured by the 'cosine' similarity metric.

The same \texttt{roberta-base} parameters are used to fit the Roberta technique from RouteLLM \citep{ong2024routellmlearningroutellms}.  The matrix factorization method assumes that 
\[\Pr(\text{GPT-4 Win}|q) = \sigma(w_2^T(v_{\text{GPT-4}}\odot(W_1^T v_q+b)- v_{\text{mixtral}}\odot(W_1^T v_q+b)) )\]
where $v_{\text{GPT-4}}$,$v_{\text{mixtral}}$ are learnable embeddings of the model of interest. We use the \texttt{text-embeddder-small-3} from OpenAI to embed the queries, and a projection dimension of $d=128$. The model is fit using Adam, with a learning rate of $3e-4$ and a weight decay of $1e-5$. To fit RoRF from not-diamond, we again use \texttt{text-embeddder-small-3} while the default parameters from Not-Diamond are used (max-depth = 20, 100 estimators).
\begin{figure}[H]
    \centering
        \includegraphics[width=0.95\textwidth]{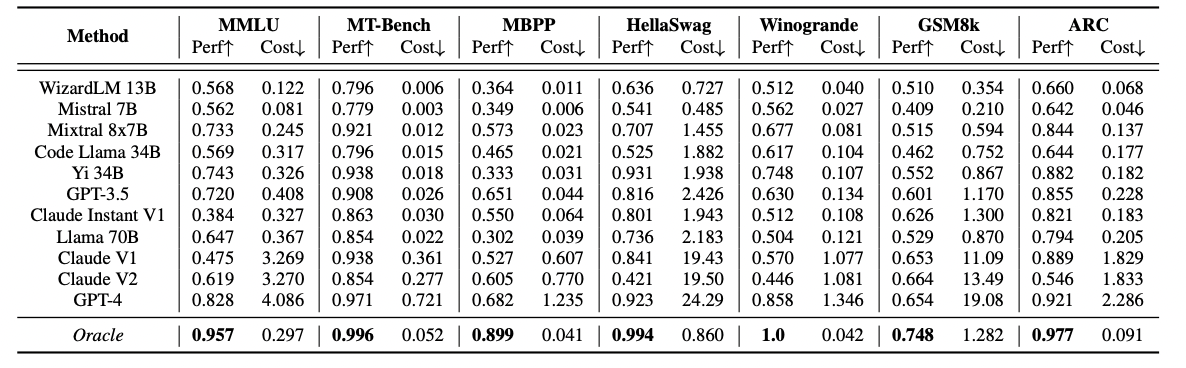} 
\caption{Routerbench models and benchmarks (\citet{hu2024routerbench} Table 1).}
\label{fig:routerbench-table}
\end{figure}

\begin{figure}[h]
    \centering
    \includegraphics[width=.45\linewidth]{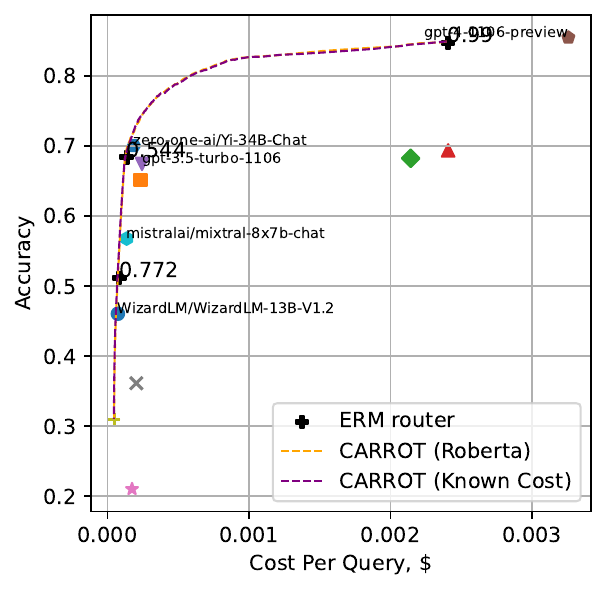}
    \caption{Router Bench Supplementary.}
    \label{fig:rbenchsupp}
\end{figure}

\subsection{Open LLM Leaderboard V2}\label{append:llms_open}

\paragraph{LLMs and costs:} Table \ref{tab:llms_open} gives all models used for the Open LLM Leaderboard experiment and their respective costs.
\begin{longtable}{lc}
\caption{Models used and their respective costs for the Open LLM Leaderboard experiment.}

\\ \hline
\textbf{Model Name} & \textbf{Price (USD per 1M tokens)} \\ \hline
NousResearch/Nous-Hermes-2-Mixtral-8x7B-DPO & 0.6 \\ \hline
01-ai/Yi-34B-Chat & 0.8 \\ \hline
Qwen/QwQ-32B-Preview & 1.2 \\ \hline
Qwen/Qwen2-72B-Instruct & 0.9 \\ \hline
Qwen/Qwen2.5-7B-Instruct & 0.3 \\ \hline
Qwen/Qwen2.5-72B-Instruct & 1.2 \\ \hline
alpindale/WizardLM-2-8x22B & 1.2 \\ \hline
deepseek-ai/deepseek-llm-67b-chat & 0.9 \\ \hline
google/gemma-2-27b-it & 0.8 \\ \hline
google/gemma-2-9b-it & 0.3 \\ \hline
google/gemma-2b-it & 0.1 \\ \hline
meta-llama/Llama-2-13b-chat-hf & 0.3 \\ \hline
meta-llama/Meta-Llama-3.1-70B-Instruct & 0.9 \\ \hline
mistralai/Mistral-7B-Instruct-v0.1 & 0.2 \\ \hline
mistralai/Mistral-7B-Instruct-v0.2 & 0.2 \\ \hline
mistralai/Mistral-7B-Instruct-v0.3 & 0.2 \\ \hline
mistralai/Mixtral-8x7B-Instruct-v0.1 & 0.6 \\ \hline
nvidia/Llama-3.1-Nemotron-70B-Instruct-HF & 0.9 \\ \hline
\label{tab:llms_open}
\end{longtable}

\paragraph{Model fitting:} The model fitting details for baseline methods are all the same as in the RouterBench experiment (following the original implementations). To fit our methods, we employ some hyperparameter tuning for both KNN and \texttt{roberta-base}. For KNN, we employ 5-fold cross-validation using ROC-AUC and the possible number of neighbors as 2, 4, 8, 16, 32, 64, 128, 256, or 512. For \texttt{roberta-base} hyperparameter tuning, we train for 3k steps, using 20\% of the training data for validation, a batch size of 8, and search for the best combination of learning rate, weight decay, and gradient accumulation steps in \{5e-5, 1e-5\}, \{1e-2, 1e-4\}, and \{1, 2, 4, 8\}. The final model is trained for 10k steps.


\begin{figure}[H]
    \centering
    \includegraphics[width=.45\linewidth]{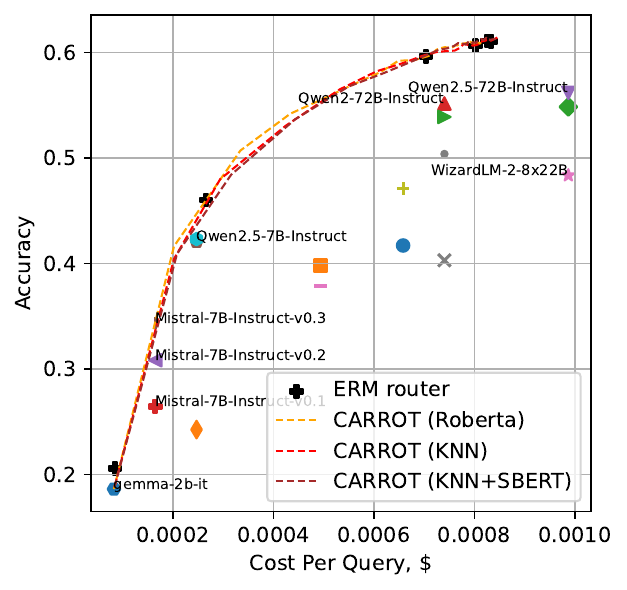}
    \caption{Open LLM leaderboard v2.}
    \label{fig:rexp_open_llm2}
\end{figure}

\section{Supplementary definitions, results and proofs}
\label{sec:proofs}

\subsection{Minimax approaches to learning the risk functions}
\label{sec:reg-fn-estimate}

In remark \ref{cor:efficient-routers} we discussed the required condition for $\widehat \Phi$ so that the plug-in router has minimax rate optimal excess risk. In this section we show that estimating $\widehat \Phi$ using \emph{local polynomial regression} (LPR) meets the requirement.  To describe the LPR estimates consider a kernel $\psi: \reals^d \to [0, \infty)$ that satisfies the regularity conditions described in the Definition \ref{def:kernel-reg} in  Appendix \ref{sec:proofs} with parameter $\max_k \gamma_k$ and define $\Theta(p)$ as the class of all $p$-degree polynomials from $\reals^d$ to $\reals$. For bandwidths $h_k > 0; k \in [K_1]$ we define the LPR estimate as 
\begin{align}\label{eq:LPR}
    [\widehat \Phi(x_0)]_{m, k} = \hat\theta_{x_0}^{(m, k)}(0); \nonumber \\
    ~~ \hat \theta_x^{(m, k)} \in \underset{\theta \in \Theta(p)}{\argmin}  \textstyle \sum_{i} \psi (\frac{X_i - x_0}{h}) \big \{[Y_i]_{m ,k} - \theta (X_i -x_0 )\big\}^2. 
\end{align}
In Theorem 3.2 of \citet{audibert2007Fast}, a similar rate of convergence for LPR estimates is established. In their case, the losses were binary. For our instance, we assume that the $Y_i$ are sub-Gaussian, but the conclusions are identical. We restate their result below.
\begin{lemma}
    Assume that $Y_i$ are sub-Gaussian random variables, \ie\ there exist constants $c_1$ and $c_2$ such that  
    \[
    \textstyle P\big ( \|Y_i\|_\infty > t \mid X\big ) \le c_1 e^{-c_2t^2}\,. 
    \] If $\psi$ is regular (\cf\ Definition \ref{def:kernel-reg}) with parameter $\max_k \gamma_k$  and $p \ge  \lfloor \max_k \gamma_k \rfloor$ then for $h_k = n^{-\nicefrac{1}{(2\gamma_k + d)}}$ the Assumption \ref{assmp:convergence} is satisfied with $a_{k, n} = n^{-\nicefrac{\gamma_k}{(2\gamma_k + d)}}$, \ie\ for some constants $\rho_1, \rho_2 > 0$ and any $n \ge 1$ and $t > 0$ and almost all $X$ with respect to $P_X$ we have the following concentration bound for $\widehat \Phi$:
    \begin{align}\label{eq:concentration-phi-2}
        \max_{P\in \cP} P \big \{ \max_{m, k} a_{k, n}^{-1}\big |[\widehat \Phi (X)]_{m, k} - [\Phi  (X)]_{m, k}\big |  \ge t\big \} \nonumber \\
        \le  \rho_1 \exp\big (- \rho_2  t^2 \big )\,.
    \end{align}
\end{lemma}
This result is related to our Remark \ref{cor:efficient-routers} about the rate-efficient estimation of routers. Estimating $\Phi(X)$ with an LPR and a suitable bandwidth and polynomial degree leads to our desired rate of convergence $a_{k, n} = n^{-\nicefrac{\gamma_k}{(2\gamma_k + d)}}$ in Assumption \ref{assmp:convergence}. 

\subsection{Examples, additional assumptions and lemmas}

Next, we describe the regularity conditions needed for local polynomial regression in eq. \eqref{eq:LPR} and \eqref{eq:LPR}. These conditions are taken directly from \citet[Section 3]{audibert2007Fast}. 
\begin{definition}[Kernel regularity]
\label{def:kernel-reg}
    For some $\beta > 0$ we say that a kernel $K:\reals^d \to [0, \infty)$ satisfies the regularity condition with parameter $\beta$, or simply $\beta$-regular if the following are true: 
    \[
    \begin{aligned}
        & \text{for some } c> 0, K(x)\ge c, ~~ \text{for } \|x\|_2 \le c\,,\\
        & \textstyle \int K(x)dx = 1\\
        & \textstyle\int (1 + \|x\|_2^{4\beta}) K^2(x)dx < \infty, \\
        & \textstyle\sup\limits_{x} (1 + \|x\|_2^{2\beta}) K(x) < \infty\,. 
    \end{aligned}
    \]
\end{definition}
An example of a kernel that satisfies these conditions is the Gaussian kernel: $K(x) = \prod_{j = 1}^d \phi(x_j)$, where $\phi$ is the density of a standard normal distribution.



Next, we establish sufficient conditions for a class of distributions $\{p_\theta, \theta\in \reals\}$ to satisfy the condition that $\KL(p_\theta, p_{\theta'}) \le K(\theta - \theta')^2$ for some $K> 0$ and any $\theta, \theta'\in \reals$. 
\begin{lemma}\label{lemma:KL-bound}
    Assume that a parametric family of distributions $\{p_\theta, \theta\in \reals\}$ satisfies the following conditions: 
    \begin{enumerate}
        \item The distributions have a density $p_\theta$ with respect to a base measure $\mu$ such that $p_\theta$ is twice continuously differentiable with respect to $\theta$. 
        \item  $\textstyle \int  \partial_\theta  p_\theta (x) d\mu(x) = \partial_\theta \int \textstyle   p_\theta (x)d\mu(x) = 0$
        \item For some $K> 0$ and all $\theta \in \reals$ the  $\textstyle - \partial_\theta^2 \int \textstyle  \log p_\theta (x) p_\theta(x)d\mu(x) \le K$.
    \end{enumerate}
    Then $\KL(p_\theta, p_{\theta'}) \le \frac{K(\theta - \theta')^2}2$.
\end{lemma}
Some prominent examples of such family are location families of normal, binomial, Poisson distributions, etc.

\begin{proof}[Proof of the Lemma \ref{lemma:KL-bound}]
    Notice that 
    \[
\begin{aligned}
    &\textstyle \KL(\mu_\theta, \mu_{\theta'})\\
    & = \textstyle \int  p_{\theta}(x) \log\big\{  \frac{p_\theta (x)}{p_{\theta'}(x)}\big\} d\mu(x) \\
    & = \textstyle \int  p_{\theta}(x)\big\{  \log p_\theta(x) - \log p_{\theta'}(x)\big \}  d\mu(x)\\
    & = \textstyle \int  p_{\theta}(x)\big\{  \log p_\theta(x) - \log p_{\theta}(x) - (\theta' - \theta) \partial_\theta \log p_\theta (x)  - \frac{(\theta' - \theta)^2}{2}\partial^2_\theta \log p_{\tilde \theta} (x) \big \}  d\mu(x)\\
\end{aligned}
\] Here, 
$\int p_\theta (x) \partial _\theta \log p_\theta (x) d\mu(x) = \int \partial _\theta  p_\theta (x) d\mu(x) dx =0$ and  $- \int p_\theta (x) \partial^2 _\theta \log p_{\tilde \theta} (x)d\mu(x) \le K$. Thus, we have the upper bound $\KL(\mu_\theta, \mu_{\theta'}) \le \frac{K}{2}(\theta - \theta')^2$. 
\end{proof}

\subsection{Proof of Lemma \ref{lemma:oracle-router}}
\begin{proof}[Proof of Lemma \ref{lemma:oracle-router}]
    The $\mu$-th risk 
    \[
    \begin{aligned}
        \textstyle \cR_P(g, \mu) & \textstyle = \Ex\big [\Ex\big[Y\mu]_m\mid X\big] \bbI\{g(X) = m\} \big]\\
        & \textstyle =\Ex\big [\big\{ \sum_{k = 1}^K \mu_k [\Phi (X)]_{m, k} \big\}\bbI\{g(X) = m\} \big]
    \end{aligned}
    \] is minimized at $g(X) = \argmin_m \big\{ \sum_{k = 1}^K \mu_k [\Phi (X)]_{m, k} \big\}$.
\end{proof}



\subsection{The upper bound}

\begin{lemma} \label{lemma:excess-risk-bound-generalization}
    Suppose that we have a function $f:\cX \to \reals ^M $ for which we define the coordinate minimizer $g:\cX \to [M]$ as $g(x) = \argmin_m f_m(x) $ and the margin function 
    \[
   \Delta(x) =  \begin{cases}
       \min_{m \neq g(x)} f_m(x) - f_{g(x)}(x) & \text{if} ~ g(x) \neq [M]\\
       0 & \text{otherwise}\,.
    \end{cases}
    \]
    Assume that the margin condition is satisfied, \ie\ there exist $\alpha, K_\alpha$ such that 
\begin{equation}
     P_X \big\{0 < \Delta (X) \le t\big \}  \le K_\alpha t^{\alpha}\,.
\end{equation} Additionally, assume that there exists an estimator $\widehat f$ of the function $f$ such that it satisfies a concentration bound: 
for some $\rho_1, \rho_2 > 0$ and any $n \ge 1$ and $t > 0$ and almost all $x$ with respect to $P_X$ we have the following concentration bound for $\widehat \Phi$:
    \begin{equation}\label{eq:concentration-f}
         P_{\cD_n} \big \{ \|\widehat f(x) - f (x)\|_\infty  \ge t\big \} \le  \rho_1 \exp\big (- \rho_2 a_n^{-2} t^2 \big )\,, 
    \end{equation}
    where $\{a_n; n \ge 1\}\subset \reals$ is a sequence that decreases to zero. Then for $\widehat g(x) = \argmin_m \widehat f_m(x)$ there exists a $K> 0$ such that for any $n \ge 1$  we have the upper bound
    \begin{equation}
         \Ex_{\cD_n}\big [ \Ex_P\big [ f_{\widehat g(X)}(X) - f_{g(X)}(X)  \big]\big ] \le K a_n^{{1+ \alpha}}\,. 
    \end{equation}
\end{lemma}

\begin{proof}
    For an $x \in \cX$ define $\delta_m(x) = f_m(x) -  f_{g(x)}(x)$. Since $g(x) = \argmin_m f_m(x)$ we have $\delta_m(x) \ge 0$ for all $m$,  $\min_{m}\delta_m(x) = 0$. Furthermore, define $h(x) = \argmin\{m \neq g(x): f_m(x)  \}$, \ie\ the coordinate of $f(x)$ where the second minimum is achieved. Clearly, $\delta_{h(x)}(x) = \Delta(x)$. With these definitions, lets break down the excess risk as: 
    \begin{equation}\label{eq:tech-1}
    \begin{aligned}
         & \Ex_{\cD_n}\big [ \Ex_P\big [ f_{\widehat g(X)}(X) - f_{g(X)}(X)  \big]\big ]\\
         & = \textstyle  \Ex_{\cD_n}\big [ \Ex_P\big [ \sum _{m = 1}^M\{f_{m}(X) - f_{g(X)}(X)\} \bbI \{ \widehat g(X) = m\}  \big]\big ] \\
         & = \textstyle  \Ex_{\cD_n}\big [ \Ex_P\big [ \sum _{m = 1}^M\{f_{m}(X) - f_{g(X)}(X)\} \bbI \{ \widehat g(X) = m\} \bbI \{ \Delta(X) \le \tau\}   \big]\big ] \\
         & \quad \textstyle + \sum_{i \ge 1} \Ex_{\cD_n}\big [ \Ex_P\big [ \sum _{m = 1}^M\{f_{m}(X) - f_{g(X)}(X)\} \bbI \{ \widehat g(X) = m\} \bbI \{ \tau 2^{i -1} < \Delta(X) \le \tau 2^i \}   \big]\big ]
    \end{aligned}
    \end{equation} where $\tau = 2\rho_2^{-\nicefrac{1}{2}}a_n$. 
    We deal with the summands one by one. First, if $\Delta(X) = 0$ then all the coordinates of $f(X)$ are identical, which further implies that $f_m(X) - f_{g(X)}(X) = 0$ for any $m$. Thus, 
    \[
    \begin{aligned}
        & \textstyle \Ex_{\cD_n}\big [ \Ex_P\big [ \sum _{m = 1}^M\{f_{m}(X) - f_{g(X)}(X)\} \bbI \{ \widehat g(X) = m\} \bbI \{ \Delta(X) \le \tau\}   \big]\big ]\\
        & \textstyle = \Ex_{\cD_n}\big [ \Ex_P\big [ \sum _{m = 1}^M\{f_{m}(X) - f_{g(X)}(X)\} \bbI \{ \widehat g(X) = m\} \bbI \{ 0 < \Delta(X) \le \tau\}   \big]\big ]
    \end{aligned}
    \] If $m = g(X)$ then the summand is zero. For the other cases, $\widehat g(X) = m $ if $\widehat f(X)$ has the minimum value at the $m$-th coordinate. This further implies $\widehat f_m (X) \le \widehat f_{ g(X)}(X)$. The only way this could happen if $|\widehat f_m(X) - f_m(X)| \ge \nicefrac{\delta_m(X)}{2}$ or $|\widehat f_{g(X)}(X) - f_{g(X)}(X)| \ge \nicefrac{\delta_m(X)}{2}$. Otherwise, if both are $|\widehat f_m(X) - f_m(X)| < \nicefrac{\delta_m(X)}{2}$ and $|\widehat f_{g(X)}(X) - f_{g(X)}(X)| < \nicefrac{\delta_m(X)}{2}$ this necessarily implies 
    \[
    \begin{aligned}
        \textstyle \widehat f_{g(X)} (X) & <\textstyle  f_{g(X)}(X) + \frac{\delta_m(X)}{2} \\
        & \textstyle = f_m(X) - \delta_m(X) + \frac{\delta_m(X)}{2}\\
        & = \textstyle f_m(X) - \frac{\delta_m(X)}{2} < \widehat f_m(X)\,, 
    \end{aligned}
    \] which means for $\widehat f(X)$ the minimum is not achieved at the $m$-th coordinate. Now, $|\widehat f_m(X) - f_m(X)| \ge \nicefrac{\delta_m(X)}{2}$ or $|\widehat f_{g(X)}(X) - f_{g(X)}(X)| \ge \nicefrac{\delta_m(X)}{2}$ implies $\|\widehat f(X) - f(X) \|_\infty \ge \nicefrac{\delta_m(X)}{2}$. With these observations we split the expectation as
    \[
    \begin{aligned}
        & \textstyle  \Ex_{\cD_n}\big [ \Ex_P\big [ \{f_{m}(X) - f_{g(X)}(X)\} \bbI \{ \widehat g(X) = m\} \bbI \{ 0 < \Delta(X) \le \tau\}   \big]\big ]\\
        & \textstyle = \Ex_{\cD_n}\big [ \Ex_P\big [ \{f_{m}(X) - f_{g(X)}(X)\} \bbI \{ \widehat g(X) = m = g(X)\} \bbI \{ 0 < \Delta(X) \le \tau\}   \big]\big ] \\
        & \textstyle \quad + \Ex_{\cD_n}\big [ \Ex_P\big [ \{f_{m}(X) - f_{g(X)}(X)\} \bbI \{ \widehat g(X) = m \neq  g(X)\} \bbI \{ 0 < \Delta(X) \le \tau\}   \big]\big ] 
    \end{aligned}
    \] The first part is zero, whereas the second part further simplifies as: 
    \[
    \begin{aligned}
        & \textstyle  \Ex_{\cD_n}\big [ \Ex_P\big [ \{f_{m}(X) - f_{g(X)}(X)\} \bbI \{ \widehat g(X) = m \neq  g(X)\} \bbI \{ 0 < \Delta(X) \le \tau\}   \big]\big ] \\
        & \le  \textstyle \Ex_{\cD_n}\big [ \Ex_P\big [ \{f_{m}(X) - f_{g(X)}(X)\} \bbI \big \{\|\widehat f(X) - f(X) \|_\infty \ge \frac{\delta_m(X)}{2} \big \} \bbI \{ 0 < \Delta(X) \le \tau\}   \big]\big ] \\
        & = \textstyle \Ex_P\big [ \{f_{m}(X) - f_{g(X)}(X)\} \Ex_{\cD_n}\big [ \bbI \big \{\|\widehat f(X) - f(X) \|_\infty \ge \frac{\delta_m(X)}{2} \big \}\big] \bbI \{ 0 < \Delta(X) \le \tau\}   \big ] \\
        & = \textstyle \Ex_P\big [ \delta_m(X) P_{\cD_n} \big \{\|\widehat f(X) - f(X) \|_\infty \ge \frac{\delta_m(X)}{2} \big \} \bbI \{ 0 < \Delta(X) \le \tau\}   \big ] \\
        & \le  \textstyle \Ex_P\big [ \delta_m(X) \rho_1 e^{- \frac{\rho_2 a_n^{-2}\delta_m^2(X)}4 } \bbI \{ 0 < \Delta(X) \le \tau\}   \big ] = \textstyle \Ex_P\big [ \delta_m(X) \rho_1 e^{- \frac{\delta_m^2(X)}{\tau^2} } \bbI \{ 0 < \Delta(X) \le \tau\}   \big ]
    \end{aligned}
    \]
    Notice that $\delta_m(X) \ge \Delta(X)$ whenever $\Delta(X) > 0$. Thus, we perform a maximization on $\delta_m(X)  e^{- \frac{\delta_m^2(X)}{\tau^2} }$ on the feasible set $\delta_m(X) \ge \Delta(X)$. Here, we use the result: 
    \begin{equation}\label{eq:tech-2}
        \max_{x \ge y} xe^{-
        \frac{x^2}{\tau^2 }} \le \begin{cases}
            \frac{\tau }{\sqrt{2e}} & \text{if} ~ \frac{\tau}{\sqrt{2}} \ge y\\
             ye^{-
        \frac{y^2}{\tau^2 }} & \text{otherwise}\,, 
        \end{cases}
    \end{equation} where $x = \delta_m(X)$ and $y = \Delta(X)$. Since $
    \Delta(X) \le \tau$ we have $\delta_m(X)  e^{- \frac{\delta_m^2(X)}{\tau^2} } \le \tau$ and thus
    \[
    \begin{aligned}
        \textstyle \Ex_P\big [ \delta_m(X) \rho_1 e^{- \frac{\delta_m^2(X)}{\tau^2} } \bbI \{ 0 < \Delta(X) \le \tau\}   \big ] \le \rho_1 \tau P\{0 < \Delta(X) \le \tau\} = \rho_1 \tau^{1 + \alpha}\,.
    \end{aligned}
    \] This finally results in 
    \[
    \begin{aligned}
        \textstyle  \Ex_{\cD_n}\big [ \Ex_P\big [ \sum _{m = 1}^M\{f_{m}(X) - f_{g(X)}(X)\} \bbI \{ \widehat g(X) = m\} \bbI \{ \Delta(X) \le \tau\}   \big]\big ] \le M\rho_1\tau^{1+\alpha}\,,
    \end{aligned}
    \] which takes care of the first summand in eq. \eqref{eq:tech-1}. Now, for an $i \ge 1 $, let us consider the summand 
    \[
    \begin{aligned}
        & \textstyle\Ex_{\cD_n}\big [ \Ex_P\big [ \sum _{m = 1}^M\{f_{m}(X) - f_{g(X)}(X)\} \bbI \{ \widehat g(X) = m\} \bbI \{ \tau 2^{i -1} < \Delta(X) \le \tau 2^i \}   \big]\big ]
    \end{aligned}
    \] Again, on the event $m = g(X)$ the the summand is zero and on the other cases we have $\|\widehat f(X) - f(X) \|_\infty \ge \nicefrac{\delta_m(X)}{2}$.
    Thus, we write 
    \[
    \begin{aligned}
        & \textstyle\Ex_{\cD_n}\big [ \Ex_P\big [ \sum _{m = 1}^M\{f_{m}(X) - f_{g(X)}(X)\} \bbI \{ \widehat g(X) = m\} \bbI \{ \tau 2^{i -1} < \Delta(X) \le \tau 2^i \}   \big]\big ]\\
        & \le  \textstyle \textstyle\sum _{m = 1}^M\Ex_{\cD_n}\big [ \Ex_P\big [ \delta_m(X)\bbI \big \{\|\widehat f(X) - f(X) \|_\infty \ge \frac{\delta_m(X)}{2} \big \} \bbI \{ \tau 2^{i -1} < \Delta(X) \le \tau 2^i \}   \big]\big ]\\
        & \le  \textstyle \textstyle\sum _{m = 1}^M\Ex_P\big [ \delta_m(X) \rho_1 e^{- \frac{\delta_m^2(X)}{\tau^2} } \bbI \{ \tau 2^{i -1} < \Delta(X) \le \tau 2^i \}   \big]\\
    \end{aligned}
    \] Because $\Delta(X) \ge \tau 2^{i-1} > \nicefrac\tau{\sqrt{2}}$ we again use the inequality in eq. \eqref{eq:tech-2} to obtain
\[
\begin{aligned}
&   \textstyle \textstyle\sum _{m = 1}^M\Ex_P\big [ \delta_m(X) \rho_1 e^{- \frac{\delta_m^2(X)}{\tau^2} } \bbI \{ \tau 2^{i -1} < \Delta(X) \le \tau 2^i \}   \big]\\
& \le  \textstyle \textstyle\sum _{m = 1}^M\Ex_P\big [ \Delta(X) \rho_1 e^{- \frac{\Delta^2(X)}{\tau^2} } \bbI \{ \tau 2^{i -1} < \Delta(X) \le \tau 2^i \}   \big]\\ 
& \le  \textstyle \textstyle\sum _{m = 1}^M \tau2^i \rho_1 e^{- \frac{\tau^2 2^{2i - 2}}{\tau^2} } P \{ \tau 2^{i -1} < \Delta(X) \le \tau 2^i \}   \\ 
& \le  \textstyle \textstyle M \tau2^i \rho_1 e^{- \frac{\tau^2 2^{2i - 2}}{\tau^2} } P \{ 0 < \Delta(X) \le \tau 2^i \} = M \rho_1 \tau^{1 + \alpha} 2^{i(1+\alpha)} e^{-2^{2i-2}}   \\ 
\end{aligned}
\] Combining all the upper bounds in \eqref{eq:tech-1} we finally obtain
\begin{equation}
    \begin{aligned}
        \Ex_{\cD_n}\big [ \Ex_P\big [ f_{\widehat g(X)}(X) - f_{g(X)}(X)  \big]\big ] \le \textstyle M\rho_1\tau^{1+\alpha} \big \{ 1+ \sum_{i \ge 1}2^{i(1+\alpha)} e^{-2^{2i-2}} \big \} 
    \end{aligned}
\end{equation} As $\sum_{i \ge 1}2^{i(1+\alpha)} e^{-2^{2i-2}}$ is finite we have the result. 
 \end{proof}



\begin{proof}[Proof of Theorem \ref{thm:upper-bound}]

    The proof of the upper bound follows directly from the lemma \ref{lemma:excess-risk-bound-generalization} once we establish that for $a_n = \sum_{k = 1}^{K_1}\mu_k a_{k, n}$ the following concentration holds: 
    for constants $\rho_{ 1}, \rho_{ 2} > 0$ and any $n \ge 1$ and $t > 0$ and almost all $X$ with respect to $P_X$ we have 
    \begin{equation} \label{eq:concentration-phi-gen}
        \max_{P\in \cP} P \big \{ \max_m \big |\widehat \eta_{\mu, m} (X) - \eta^\star_{\mu, m}  (X)\big |  \ge t\big \} \le  \rho_1 \exp\big (- \rho_2 a_{ n}^{-2} t^2 \big )\,.   
    \end{equation} To this end, notice that 
    \[
    \begin{aligned}
        & \textstyle\max_m \big |\widehat \eta_{\mu, m} (X) - \eta_{\mu, m}  (X)\big | \\
        & \textstyle \le \sum_{k = 1}^{K} \mu_k  \max_m\big |[\widehat \Phi(X)]_{m, k} - [ \Phi(X)]_{m, k} \big| \\
        & = \textstyle \sum_{k = 1}^{K_1} \mu_k  \max_m\big |[\widehat \Phi(X)]_{m, k} - [ \Phi(X)]_{m, k} \big|
    \end{aligned}
    \] where the last equality holds because $[\widehat \Phi(X)]_{m, k} = [\Phi(X)]_{m, k}$ for $k \ge K_1+1$. Following this inequality, we have that for any $P\in \cP$ \[
    \begin{aligned}
        & \textstyle P \big \{ \max_m \big |\widehat \eta_{\mu, m} (X) - \eta_{\mu, m}  (X)\big |  \ge K_1 t\big \}\\
        & \le \textstyle \sum_{k = 1}^{K_1} P \big \{ \max_m \big |[\widehat \Phi(X)]_{m, k} - [ \Phi(X)]_{m, k} \big|  \ge \frac t{\mu_k}\big \} \\
        & \le \textstyle \sum_{k = 1}^{K_1}\rho_{k, 1} \exp\big (- \rho_{k, 2} \mu_k ^{-2}a_{ k, n}^{-2} t^2 \big )\\
        & \textstyle \le  \rho_1 \exp\big (- {\rho_{ 2}}{K_1^2} \{\wedge_{k = 1}^{K_1}\mu_k^{-1}a_{ k, n}^{-1}\}^2 t^2 \big )
    \end{aligned}
    \] where $\rho_1 = \frac{\max_{k \le K_1} \rho_{k, 1}}{K_1}$ and $\rho_2 = K_1^{-2} \times \{ \wedge_{k \le K_1} \rho_{k, 2}\} $. Note that 
    \[
    \textstyle K_1\{\wedge_{k = 1}^{K_1}\mu_k^{-1}a_{ k, n}^{-1}\}^{-1} = K_1 \max_{k = 1}^{K_1}\mu_k a_{ l, n} \ge  \sum_{k \le K_1 } \mu_k a_{ k, n} = a_n\,. 
    \] Thus, 
    \[
    \begin{aligned}
       & \textstyle P \big \{ \max_m \big |\widehat \eta_{\mu, m} (X) - \eta_{\mu, m}  (X)\big |  \ge K_1 t\big \}\\
        & \textstyle \le  \rho_1 \exp\big (- {\rho_{ 2}}{K_1^2} \{\wedge_{k = 1}^{K_1}\mu_k^{-1}a_{ k, n}^{-1}\}^2 t^2 \big ) \le \rho_1 \exp\big (- {\rho_{ 2}}a_n ^2 t^2 \big ) \,.
    \end{aligned}
    \]

\end{proof}

\subsection{The lower bound}
To begin, we discuss the high-level proof strategy that will achieve our lower bound. Ultimately, for every $k \le K_1$ we shall establish that for any $\eps_k \in [0, 1]$ and $n \ge 1$
\begin{equation} \label{eq:lower-bound-individual}
     \textstyle  \min\limits_{A_n \in \cA_n} \max\limits_{P \in \cP} ~~ \cE_P(\mu, A_n) \ge c_k \big \{ \mu_k n^{- \frac{\gamma_k}{2\gamma_k + d}}\big\}^{1+\alpha} \,,
\end{equation} for some constant $c_k > 0$. Then, defining $c = \min\{c_k: k \le K_1\}$ we have the lower bound
\[
\begin{aligned}
    \textstyle  \min\limits_{A_n \in \cA_n} \max\limits_{P \in \cP} ~~ \cE_P(\mu, A_n) & \ge \textstyle  \max\limits_{k \le K_1 } c_k \big \{ \mu_k n^{- \frac{\gamma_k}{2\gamma_k + d}}\big\}^{1+\alpha}\\
    & \ge \textstyle  \max\limits_{k \le K_1 } c \big \{ \mu_k n^{- \frac{\gamma_k}{2\gamma_k + d}}\big\}^{1+\alpha}\\
    & \ge \textstyle  c  \big \{ \sum_{k \le K_1 }\frac{\mu_k n^{- \frac{\gamma_k}{2\gamma_k + d}}}{K}\big\}^{1+\alpha}\\
    & \ge \textstyle  c K^{-1-\alpha}  \big \{ \sum_{k \le K_1 }{\mu_k n^{- \frac{\gamma_k}{2\gamma_k + d}}}\big\}^{1+\alpha}\,,\\
\end{aligned}
\] which would complete the proof. 

It remains to establish \eqref{eq:lower-bound-individual} for each $k \in [K_1]$. To obtain this, we construct a finite family of probability measures $\cM_r \subset \cP$ (indexed by $[r]$) and study $\max\limits_{P \in \cM_r}$. The technical tool which allows this to be fruitful is a generalized version of Fano's lemma. 

\begin{lemma}[Generalized Fano's lemma]
\label{lemma:fano}
Let $r \ge 2$ be an integer and let $\cM_r \subset \cP$ contains $r$ probability measures indexed by $\{1, \dots , r\}$ such that for a pseudo-metric $d$ (\ie\ $d(\theta , \theta') = 0$ if and only if $\theta = \theta'$) any $j \neq j'$
\[
\textstyle d\big (\theta(P_j), \theta(P_{j'})\big) \ge \alpha_r , ~~\text{and} ~~ \text{KL}(P_j , P_{j'}) \le \beta_r\,.
\] Then 
\[
\textstyle \max\limits_j \Ex_{P_j}\big[ d(\theta(P_j), \widehat \theta)\big ] \ge \frac{\alpha_r}{2} \big (1 - \frac{\beta_r + \log 2}{\log r}\big )\,. 
\]
    
\end{lemma} In our construction $\theta(P^\sigma)  = g^\star_{\mu, \sigma}$ and $d\big (\theta(P^{\sigma_0}), \theta(P^{\sigma_1})\big) = \cE_{P^{\sigma_0}}(g^\star_{\mu, \sigma_{1}}, \mu)$.

\input{sections/lower_bound_supp}



\section{Judge Prompt}
\label{sec: JP}

\begin{lstlisting}[language={},title=Prompt format for the Llama 3.1 70b evaluator instruct LLM,captionpos=top]
<dmf>user
I want you to act as a judge for how well a model did answering a user-defined task. You will be provided with a user-defined task that was given to the model, its golden answer(s), and the model's answer. The context of the task may not be given here. Your task is to judge how correct is the model's answer. Your task is to judge how correct the model's answer is based on the golden answer(s), without seeing the context of the task, and then give a correctness score. The correctness score should be one of the below numbers: 0.0 (totally wrong), 0.1, 0.2, 0.3, 0.4, 0.5, 0.6, 0.7, 0.8, 0.9, or 1.0 (totally right). You should also add a brief justification regarding how the model's answer conforms to or contradicts the golden answer(s). 

Your response must follow the format
{{
   "correctness_score": your_score, 
   "justification": your_justification
}}

Note that each one of the golden answers is considered correct. Thus if the Model's Answer matches anyone of the golden answers, it should be considered correct.
<dmf>assistant
Certainly! I will act as a judge and adhere to the format you requested. Under oath, I will judge without consulting my internal knowledge of facts under any circumstances.
<dmf>user
User-defined task --- Sandy bought 1 million Safe Moon tokens. She has 4 siblings. She wants to keep half of them to herself and divide the remaining tokens among her siblings. After splitting it up, how many more tokens will she have than any of her siblings?

Golden Answer(s) --- \<answer 1\> 375000

Model's Answer --- Sandy will have more tokens than any sibling by 3/8 million.
<dmf>assistant
{{
    "correctness_score": 1.0,
    "justification": "The golden answer states that Sandy will have 375,000 more tokens than any of her siblings, which is a precise numerical value. The model's answer translates this scenario into a fraction of the total, saying Sandy will have more tokens than any sibling by 3/8 million. 1 million tokens * 3/8 =375,000 tokens. So the model provided an answer in fractional form that, when converted to a numerical value, exactly matches the golden answer's quantity."
}}
<dmf>user
User-defined task --- Extract the names and emails mentioned in the following text. Richard Matthew Stallman, also known by his initials, rms,[1] is an American free software movement activist and programmer. He can be reached at rms@gnu.org based on the information on his webpage.

Golden Answer(s) --- {{"names":  ["Richard Matthew Stallman"], "emails": ["rms@gnu.org"]}}

Model's Answer --- names: Richard Matthew Stallman
<dmf>assistant
{{
    "correctness_score": 0.5,
    "justification": "The model correctly identified the names but failed to extract the emails."
}}
<dmf>user
User defined task --- Question: In 2015 Edgar Lungu became prime minister of?

Golden Answer: 
<answer 1> Zambia; <answer 2> Zamibia; <answer 3> People of Zambia; <answer4> Zambian cuisine; <answer 5> Zambians; <answer 6> Culture of Zambia; <answer 7> Etymology of Zambia; <answer 8> Zambia; <answer 9> Health care in Zambia; <answer 10> ISO 3166-1:ZM; <answer11> Republic Of Zambia; <answer 12> Cuisine of Zambia; <answer 13> Sport in Zambia; <answer 14> Republic of Zambia; <answer 15> Zambian people; <answer 16> Name of Zambia

Model's Answer: Prime Minister
<dmf>assistant
{{
    "correctness_score": 0.0,
    "justification": "The golden answers provide a detailed list of entities all relating to Zambia, indicating that Edgar Lungu became the leader (specifically, they mentioned \"prime minister\") of Zambia in 2015. The model's answer, \"Prime Minister,\" merely repeats part of the question without answering it."
}}
<dmf>user
User defined task --- Give a one-line concise summary of the following news
PARIS (Reuters) - French President Emmanuel Macron will set out plans for reforming the European Union on Tuesday, including proposals for a separate eurozone budget, despite a German election result that is likely to complicate his far-reaching ambitions. German Chancellor Angela Merkel s conservatives saw their support slide in Sunday s election, though they remain the biggest parliamentary bloc. She is expected to seek a coalition with the liberal Free Democrats (FDP) - who have criticized Macron s ideas for Europe - and the Greens. Elysee officials said Macron, who has promised sweeping reforms to Europe s monetary union in coordination with Merkel, hoped the issues to be raised in his speech would be taken into account in Germany s coalition negotiations. One Elysee official said a eurozone budget, one of Macron s most contentious ideas, would be necessary in due course and that the president would therefore raise the issue in his speech, to be delivered at the Sorbonne University in Paris. Since his election in May, Macron has made the overhaul of the EU and its institutions one of his major themes. As well as his eurozone budget idea, he wants to see the appointment of a eurozone finance minister and the creation of a rescue fund that would preemptively help countries facing economic trouble. Ahead of Sunday s election, Merkel had indicated her willingness to work with Macron on a reform agenda, even if her own ideas may not reach as far as his. But the election results have left Merkel facing a difficult coalition-building task which is in turn likely to limit her flexibility on Europe. A coalition of Merkel s CDU/CSU bloc, the FDP and the Greens is unprecedented at the national level - and any attempt by the chancellor and Macron to press for greater EU integration will face opposition from the new German lower house Bundestag. The FDP has called for a phasing out of Europe s ESM bailout fund and changes to EU treaties that would allow countries to leave the euro zone. And the far-right, eurosceptic Alternative for Germany is now the third biggest party in the Bundestag, further curbing Merkel s room for maneuver. But Elysee officials noted that the FDP had reaffirmed its attachment to the EU and to strong Franco-German relations, a point Macron was likely to emphasize in his speech while at the same time not seeking to impose anything on his partners. Macron, the sources said, would propose that the whole EU move forward together, and that those who did not want to should not stand in the way of those that did. Coming just two days after the German election, Macron s speech is likely to be interpreted in Germany as an attempt to shape the debate before the coalition talks begin in earnest. German coalition agreements are strict, with the contours set out in them limiting the government s room for maneuver. 
In that respect, Merkel s ability to work with Macron on EU reform will be pre-determined by whatever coalition deal is struck. FDP leader Christian Lindner said on Monday he would not agree to any coalition that did not promise a change in the German government s direction. While that appeared to set the stage for tough talks, he also offered hope for Macron, saying when asked about the French president s eurozone budget ideas that the FDP had a strong interest in the strength of France.

Golden Answer(s) --- After German election, Macron to set out his vision for Europe

Model's Answer --- French President Emmanuel Macron to introduce plans for reforming the European Union amid the uncertain aftermath of German elections
<dmf>assistant
{{
    "correctness_score": 0.6,
    "justification": "While the golden answer is more concise, the model's answer is largely similar to the golden answer in its meaning."
}}
<dmf>user
User defined task --- {MODELTASK}
Golden Answer(s) --- {GOLDENRESPONSE}
Model's Answer --- {MODELRESPONSE}
\end{lstlisting}

\section*{NeurIPS Paper Checklist}



\begin{enumerate}

\item {\bf Claims}
    \item[] Question: Do the main claims made in the abstract and introduction accurately reflect the paper's contributions and scope?
    \item[] Answer: \answerYes{} 
    \item[] Justification: The claims made match theoretical and experimental results.

\item {\bf Limitations}
    \item[] Question: Does the paper discuss the limitations of the work performed by the authors?
    \item[] Answer: \answerNo{} 
    \item[] Justification: While our work has limitations, we did not highlight any specifics in the conclusion or main text.

\item {\bf Theory assumptions and proofs}
    \item[] Question: For each theoretical result, does the paper provide the full set of assumptions and a complete (and correct) proof?
    \item[] Answer: \answerYes{} 
    \item[] Justification: Proofs to theoretical claims are included in the appendix.

    \item {\bf Experimental result reproducibility}
    \item[] Question: Does the paper fully disclose all the information needed to reproduce the main experimental results of the paper to the extent that it affects the main claims and/or conclusions of the paper (regardless of whether the code and data are provided or not)?
    \item[] Answer: \answerYes{} 
    \item[] Justification: Data, baselines, and our method are each carefully described.

\item {\bf Open access to data and code}
    \item[] Question: Does the paper provide open access to the data and code, with sufficient instructions to faithfully reproduce the main experimental results, as described in supplemental material?
    \item[] Answer: \answerYes{} 
    \item[] Justification: Code included in supplementary material, CARROT and \newdata\ are available online.

\item {\bf Experimental setting/details}
    \item[] Question: Does the paper specify all the training and test details (e.g., data splits, hyperparameters, how they were chosen, type of optimizer, etc.) necessary to understand the results?
    \item[] Answer: \answerYes{} 
    \item[] Justification: Details are included in the appendix

\item {\bf Experiment statistical significance}
    \item[] Question: Does the paper report error bars suitably and correctly defined or other appropriate information about the statistical significance of the experiments?
    \item[] Answer: \answerNo{} 
    \item[] Justification: We present results for the version of CARROT we release to the public.

\item {\bf Experiments compute resources}
    \item[] Question: For each experiment, does the paper provide sufficient information on the computer resources (type of compute workers, memory, time of execution) needed to reproduce the experiments?
    \item[] Answer: \answerNo{} 
    \item[] Justification: The compute needed to run the experiments is minimal (one simply needs to fit KNN or train a small transformer). We also release our data and models so that one can check our results with no training needed. 
    
\item {\bf Code of ethics}
    \item[] Question: Does the research conducted in the paper conform, in every respect, with the NeurIPS Code of Ethics \url{https://neurips.cc/public/EthicsGuidelines}?
    \item[] Answer: \answerYes{} 
    \item[] Justification: No ethics violations are present in this work.

\item {\bf Broader impacts}
    \item[] Question: Does the paper discuss both potential positive societal impacts and negative societal impacts of the work performed?
    \item[] Answer: \answerNA{} 
    \item[] Justification: We feel that there are no specific societial impacts of this work that must be highlighted.
    
\item {\bf Safeguards}
    \item[] Question: Does the paper describe safeguards that have been put in place for responsible release of data or models that have a high risk for misuse (e.g., pretrained language models, image generators, or scraped datasets)?
    \item[] Answer: \answerNA{}
    \item[] Justification: The released models and datasets do not pose such risks.

\item {\bf Licenses for existing assets}
    \item[] Question: Are the creators or original owners of assets (e.g., code, data, models), used in the paper, properly credited and are the license and terms of use explicitly mentioned and properly respected?
    \item[] Answer: \answerYes{} 
    \item[] Justification: We cite all papers that produce models or datasets we used.

\item {\bf New assets}
    \item[] Question: Are new assets introduced in the paper well documented and is the documentation provided alongside the assets?
    \item[] Answer: \answerYes{} 
    \item[] Justification: We describe how CARROT is designed and how \newdata\ was collected.

\item {\bf Crowdsourcing and research with human subjects}
    \item[] Question: For crowdsourcing experiments and research with human subjects, does the paper include the full text of instructions given to participants and screenshots, if applicable, as well as details about compensation (if any)? 
    \item[] Answer: \answerNA{}. 
    \item[] Justification: \justificationTODO{}

\item {\bf Institutional review board (IRB) approvals or equivalent for research with human subjects}
    \item[] Question: Does the paper describe potential risks incurred by study participants, whether such risks were disclosed to the subjects, and whether Institutional Review Board (IRB) approvals (or an equivalent approval/review based on the requirements of your country or institution) were obtained?
    \item[] Answer: \answerNA{}. 
    \item[] Justification: \justificationTODO{}

\item {\bf Declaration of LLM usage}
    \item[] Question: Does the paper describe the usage of LLMs if it is an important, original, or non-standard component of the core methods in this research? Note that if the LLM is used only for writing, editing, or formatting purposes and does not impact the core methodology, scientific rigorousness, or originality of the research, declaration is not required.
    \item[] Answer: \answerNA{}. 
    \item[] Justification: \justificationTODO{}

\end{enumerate}

%% file: sections/lower_bound_supp.tex
Next, we lay out the template for constructing the family $\cM_r$. Fix a $k_0 \in [K_1]$ and define the following. 
\begin{definition}
    \begin{enumerate}
    \item For an $h = L \times   \mu_{k_0}^{\frac{1}{\gamma_{k_0}}} n^{-\frac{1}{2\gamma_{k_0} + d}}$ ($L > 0$ is a constant to be decided later) define $m =  \lfloor h^{-1} \rfloor $.   
    \item  Define $\cG = \big[\{ ih + \frac{h}{2} : i = 0,  \dots, m - 1\}^d\big]$ as a uniform grid in $[0, 1]^d$ of size $m^d$ and $\cG_\eps$ as an $\epsilon$-net in $\ell_\infty$ metric, \ie\  $\cG_\eps = \cup_{x \in \cG} \cB(x, \eps, \ell_\infty) $, where $\cB(x, \eps, \ell_{\infty}) = \{y \in \cX : \|x-y\|_\infty \le \eps\}$.
    \item Define $P_X = \text{Unif}(\cG_\eps) $. For such a distribution, note that  $\vol (\cG_\eps) = (m\eps)^d \le (h^{-1}\eps)^d $, which implies that for all $x \in \cG_\eps$ we have $p_X(x) = (h\eps^{-1})^d $. Setting $\eps = p_0 ^{-\nicefrac{1}{d}} h \wedge \frac{h}{3}$ we have $p_X(x) \ge p_0$ that satisfies the strong density assumption for $P_X$.
    \item  Fix an $m_0 \le m^d$ and consider $\cG_0 \subset \cG$ such that $|\cG_0| =m_0$ and define $\cG_1 = \cG \backslash \cG_0$.
    \item  For a function $\sigma: \cG_0 \to [M]$ define \begin{equation}
    \Phi_{m, k}^{\sigma}(x) = \begin{cases}
     \frac{1 - K_{\gamma, k_0} \mu_{k_0}^{-1} \eps^{\gamma_{k_0}}\bbI\{\sigma(y) = m\}}{2} & \text{when} ~ k =k_0, ~  x \in\cB(x, \eps, \ell_{\infty})~ \text{for some}   ~  y \in \cG_0,\\
     \frac{1}{2} & \text{elsewhere.}
    \end{cases}
\end{equation}
\item    Consider a class of probability distributions $\{\mu_\theta: \theta \in \reals\}$ defined on the same support $\text{range}(\ell)$ that have mean $\theta$ and satisfy $\text{KL}(\mu_\theta, \mu_{\theta'})  \le c  (\theta - \theta')^2$ for some $c > 0$. A sufficient condition for constricting such a family of distributions can be found in Lemma \ref{lemma:KL-bound}. Some prominent examples of such family are location families of normal, binomial, Poisson distributions, etc.  Define the probability $P^\sigma([Y]_{m, k} \mid X = x) \sim \mu_{\Phi_{m, k}^{(\sigma)}(x)}$. 
\end{enumerate}
\end{definition}
The following two lemmas (along with the observation on the strong density condition) will establish that for a given $\sigma$, the distribution over $\cX, \cY$ given by $P^\sigma([Y]_{m, k} \mid X = x) \times \text{Unif}[\cG_{\epsilon}]$ is indeed a member of the class $\cP$.
\begin{lemma}
    \label{lem: margin condition}
    Fix a choice for $\sigma$ and let $\eta_{\mu, m}^\sigma = \sum_k \mu_k \Phi_{k,m}^{\sigma}(x)$, then $\eta_{\mu, m}^\sigma $ satisfies $\alpha$-margin condition.
\end{lemma}
\begin{proof}
To see that $\eta_{\mu, m}^\sigma $ satisfies $\alpha$-margin condition, notice that
\[
\begin{aligned}
    & \textstyle \eta_{\mu, m}^\sigma(x) = \begin{cases}
     \frac{1 - K_{\gamma, k_0} \eps^{\gamma_{k_0}} \bbI\{\sigma(y) = m\}}{2} & \text{when} ~  x \in\cB(x, \eps, \ell_{\infty})~ \text{for some}   ~  y \in \cG_0,\\
     \frac{1}{2} & \text{elsewhere.}
    \end{cases}
\end{aligned}
\]
Thus, for every $x \in \cB(y, \eps, \ell_\infty), y \in \cG_0$ the $\Phi_{\mu, m}^\sigma(x) = \frac12$ for all but one $m$ and at $m = \sigma(x)$ the $\Phi_{\mu, m}^\sigma(x) = \frac{1 - K_{\gamma, k_0} \eps^{\gamma_{k_0}} }{2}$, leading to $\Delta_\mu^\sigma (x)  = \frac{K_{\gamma, k_0} \eps^{\gamma_{k_0}} }{2}$ at those $x$, and at all other $x$ we have $\Delta_\mu^\sigma(x) =0$. This further implies $P_X(0 < \Delta_\mu^\sigma(X) \le t) = 0 $ whenever $t < \frac{K_{\gamma, k_0} \eps^{\gamma_{k_0}} }{2} $ and for $t \ge \frac{K_{\gamma, k_0} \eps^{\gamma_{k_0}} }{2} $ we have 
\[
\begin{aligned}
   P_X(0 < \Delta^\sigma(X) \le t) & \textstyle = P_X\big ( \Phi_m^\sigma(X) \neq \frac12 \text{ for some } m \in [M]\big ) \\
   & \textstyle \le m_0 \eps^d \le K_\alpha\big(\frac{K_{\gamma, k_0} \eps^{\gamma_{k_0}} }{2}\big)^\alpha 
\end{aligned}
\] whenever
\[
\begin{aligned}
    m_0  \textstyle \le  K_\alpha 2^{-\alpha}   K_{\gamma, k_0}^\alpha \eps^{ \alpha \gamma_{k_0} -d}
\end{aligned}
\] We set $m_0 = \lfloor K_\alpha 2^{-\alpha}   K_{\gamma, k_0}^\alpha \eps^{ \alpha \gamma_{k_0} -d} \rfloor$ to meet the requirement.
Since $d > \min_{k} \alpha\gamma_{k}$, for sufficiently small  $\eps$ we have $m_0 \ge 8$. 
\end{proof}
\begin{lemma}
\label{lem: Holder}
   On the support of $P_X$ the $\Phi_{m, k}^\sigma$ are $(\gamma_{k}, K_{\gamma, k})$ H\"older smooth. 
\end{lemma}
\begin{proof}
 Note that the only way $\Phi_{m, k}^\sigma(x)$ and $\Phi_{m,k}^\sigma(x')$ can be different if $\|x- x'\|_\infty  \ge \frac{h}{3}$. Since $\eps \le \frac h3$ for such a choice, we have 
\[
\begin{aligned}
    \textstyle |\Phi_{m, k}^\sigma(x) - \Phi_{m, k}^\sigma(x') | & \textstyle \le \frac12 K_{\gamma, k} \eps ^\beta\\
    & \textstyle \le  K_{\gamma, k} (\frac h3) ^\beta\\
    & \textstyle \le K_{\gamma, k} \|x - x'\|_\infty ^\beta \le K_{\gamma, k} \|x - x'\|_2^\beta\,.
\end{aligned}
\]
\end{proof}
In order transfer the inequality in Fano's lemma to a statement on rate of convergence, we need an upper bound on $ \text{KL}(P^{\sigma_1}, P^{\sigma_2})$ and a lower bound on the semi-metric $\cE_{P^{\sigma_0}}(\mu, g^\star_{\mu, \sigma_1})$. These are established in the next two lemmas.
\begin{lemma}
   Consider the probability distribution $P^\sigma$ for the random pair $(X, Y)$ where $X \sim P_X$ and given $X$ the $\{[Y]_{m, k}; m \in [M], k \le K_1\}$ are all independent and distributed as $[Y]_{m, k} \mid X = x \sim \mu_{\Phi_{m, k}^{\sigma}(x)}$. Let $C$ be a positive constant and $\delta(\sigma_1, \sigma_2) = \sum_{y \in \cG_0 }   \bbI\{\sigma_1(y) \neq \sigma_2(y) \} $ the Hamming distance between $\sigma_1$ and $\sigma_2$. Then following upper bound holds on $\text{KL}(P^{\sigma_1}, P^{\sigma_2})$. 
   \[ \text{KL}(P^{\sigma_1}, P^{\sigma_2}) \leq C \mu_{k_0}^{-2} h ^{2\gamma_{k_0} + d }\delta(\sigma_1 , \sigma_2)\]
\end{lemma}
\begin{proof}
\[
\begin{aligned}
    & \text{KL}(P^{\sigma_1}, P^{\sigma_2}) & \\
    & = \textstyle \int dP_X(x) \sum_{m=1}^M \sum_{k = 1 }^K\text{KL}\big(\mu_{\Phi_{m, k}^{(\sigma_1)}(x)} , \mu_{\Phi_{m,k}^{(\sigma_2)}(x)} \big)  & \\
    & \le \textstyle \int dP_X(x) \sum_{m=1}^M \sum_{k = 1 }^Kc\big({\Phi_{m, k}^{(\sigma_1)}(x)} - {\Phi_{m,k}^{(\sigma_2)}(x)} \big)^2 \quad \quad (\text{KL}(\mu_\theta, \mu_{\theta'})  \le c  (\theta - \theta')^2)\\
    & = \textstyle \sum_{y \in \cG_0 } \eps^d\sum_{m=1}^M  \frac{cK_{\gamma, k_0}^2\eps^{2\gamma_{k_0}}\mu_{k_0}^{-2} }{4}\big(\bbI\{\sigma_1(y) = m\}  - \bbI\{\sigma_2(y) = m\} \big)^2\\
    & \le \textstyle \frac{c K_{\gamma, k_0}^2 }{4}\sum_{y \in \cG_0 }  \mu_{k_0}^{-2} \eps ^{2\gamma_{k_0} + d }\times \bbI\{\sigma_1(y) \neq \sigma_2(y) \}  \\
    & \textstyle \le   C \mu_{k_0}^{-2} h ^{2\gamma_{k_0} + d }\delta(\sigma_1 , \sigma_2) \quad \quad  \textstyle (\text{because} ~~\eps \le \frac{h}{3})
\end{aligned}
\] for some $C>0$, where $\delta(\sigma_1, \sigma_2) = \sum_{y \in \cG_0 }   \bbI\{\sigma_1(y) \neq \sigma_2(y) \} $ is the Hamming distances between $\sigma_1$ and $\sigma_2$. 
\end{proof}
Now, we establish a closed form for the excess risk 
\[
\textstyle \cE_{P^{\sigma_0}}(\mu, g^\star_{\mu, \sigma_1}) = \Ex_{P^{\sigma_0}}(\mu, g^\star_{\mu, \sigma_1}) - \Ex_{P^{\sigma_0}}(\mu, g^\star_{\mu, \sigma_0})
\]
where $g^\star_{\mu, \sigma_0}$ is the Bayes classifier for $P^{\sigma_0}$ defined as $g^\star_{\mu, \sigma_0}(x) = \argmin_m \Phi_{\mu, m}^{\sigma_0}(x)$. 
\begin{lemma}
    Let $\delta(\sigma_0, \sigma_1)$ denote the Hamming distance between $\sigma_0$ and $\sigma_1$ as before. Then
    \[\cE_{P^{\sigma_0}}(\mu, g^\star_{\mu, \sigma_1}) = \textstyle \frac{ K_{\gamma, k_0} \eps^{\gamma_{k_0} + d} \delta(\sigma_0, \sigma_1)}{2} \]
\end{lemma}
\begin{proof}
    
For the purpose, notice that
\[
\textstyle g^\star_{\mu, \sigma}(x) = 
    \sigma(y)  ~~\text{whenever} ~ x \in\cB(x, \eps, \ell_{\infty})~ \text{for some}   ~  y \in \cG_0\,.
\] This further implies 
\[
\begin{aligned}
     & \Ex_{P^{\sigma_0}}(\mu, g^\star_{\mu, \sigma_1})\\
    & \textstyle = \int dP_X(x) \sum_{m = 1}^M \bbI\{ g^\star_{\mu, \sigma_1}(x) = m \} \Phi_{\mu, m}^{\sigma_0}(x)\\
    & \textstyle = \sum_{y \in \cG_0} \eps^d \sum_{m = 1}^M  \bbI\{ \sigma_1(y) = m \}  \mu_{k_0} \frac{1}{2} \big \{1 - K_{\gamma, k_0} \mu_{k_0}^{-1} \eps^{\gamma_{k_0}} \bbI\{ \sigma_0(y) = m \} \big\} \\
    & \textstyle \quad  + \sum_{y \in \cG_0} \eps^d \sum_{m = 1}^M \bbI\{ \sigma_1(y) = m \} \sum_{k \neq k_0}\frac{\mu_k}{2} + \sum_{y \in \cG_1} \eps^d \sum_{m = 1}^M \bbI\{ \sigma_1(y) = m \} \frac{1}{2} \\
    & \textstyle =  -\sum_{y \in \cG_0}  \sum_{m = 1}^M    \frac{ K_{\gamma, k_0} \eps^{\gamma_{k_0} + d}}{2}   \bbI\{ \sigma_0(y) = \sigma_1(y) = m \} \\
    & \textstyle \quad + \sum_{y \in \cG_0\cup \cG_1} \eps^d \sum_{m = 1}^M \bbI\{ \sigma_1(y) = m \} \frac{1}{2}\\
    & \textstyle =  -\sum_{y \in \cG_0}  \sum_{m = 1}^M  \frac{ K_{\gamma, k_0} \eps^{\gamma_{k_0} + d}}{2}   \bbI\{ \sigma_0(y) = \sigma_1(y) = m \} +  \sum_{y \in \cG_0\cup \cG_1}   \frac{\eps^d}{2}\\
\end{aligned}
\] 
By replacing $\sigma_1$ with $\sigma_0$ in the above calculations we obtain 
\[
\textstyle\Ex_{P^{\sigma_0}}(\mu, g^\star_{\mu, \sigma_0})= -\sum_{y \in \cG_0}  \sum_{m = 1}^M  \frac{ K_{\gamma, k_0} \eps^{\gamma_{k_0} + d}}{2}   \bbI\{ \sigma_0(y) =  m \} +  \sum_{y \in \cG_0\cup \cG_1}   \frac{\eps^d}{2}
\]
and hence 
\[
\begin{aligned}
    & \cE_{P^{\sigma_0}}(g^\star_{\mu, \sigma_1}, \mu)\\
    & = \Ex_{P^{\sigma_0}}(g^\star_{\mu, \sigma_1}, \mu) - \Ex_{P^{\sigma_0}}(g^\star_{\mu, \sigma_0}, \mu)\\
    & = \textstyle \sum_{y \in \cG_0}  \sum_{m = 1}^M  \frac{ K_{\gamma, k_0} \eps^{\gamma_{k_0} + d}}{2}  \big\{ \bbI\{ \sigma_0(y) =  m \} -  \bbI\{ \sigma_0(y)  =\sigma_1(y) = m \}\big\}\\
    & = \textstyle \frac{ K_{\gamma, k_0} \eps^{\gamma_{k_0} + d}}{2}\sum_{y \in \cG_0}  \sum_{m = 1}^M   \bbI\{ \sigma_0(y) =  m \}\times \bbI\{ \sigma_1(y) \neq  m \} \\
    & = \textstyle \frac{ K_{\gamma, k_0} \eps^{\gamma_{k_0} + d}}{2}\sum_{y \in \cG_0}    \bbI\{ \sigma_0(y) \neq \sigma_1(y) \} \\
    & = \textstyle \frac{ K_{\gamma, k_0} \eps^{\gamma_{k_0} + d} \delta(\sigma_0, \sigma_1)}{2}   \,.
\end{aligned}
\] 
\end{proof}
The final technical ingredient we require is the Gilbert–Varshamov bound for linear codes. 
\begin{lemma}[Gilbert–Varshamov bound]
\label{lemma:Gilbert–Varshamov}
   Consider the maximal $A_M(m_0, d) \subset [M]^{m_0}$ such that each element in $C$ is at least $d$ Hamming distance from each other, \ie\ for any $\sigma_1 , \sigma_2\in C$ we have $\delta(\sigma_1 , \sigma_2) \ge d$. Then 
   \[
 \textstyle   |A_M(m_0, d)| \ge \frac{M^{m_0}}{\sum_{i = 0}^{d-1} {m_0 \choose i} (M-1)^i }
   \]
    Furthermore, when $M\ge 2$ and $0 \le p \le 1 - \frac1M$ we have $|A_M(m_0, pm_0 )| \ge M^{m_0 (1 - h_{M}(p))}$ where $h_{M}(p) =  \frac{p \log(M - 1) - p\log p - (1 - p)\log(1 - p)}{\log M}$. 
\end{lemma}
\begin{proof}[Proof of the Theorem \ref{thm:lower-bound}]

    
For the choice $p = \frac14$ we have $- p\log p - (1 - p)\log(1 - p) \le \frac{1}{4}$ and thus 
\[
\textstyle h_{M}(p) \le \frac{\log(M-1)}{4\log M} + \frac{1}{4\log M} \le \frac14 + \frac{1}{4\log 2} \le \frac{3}{4}\,.
\]
Consequently, the lemma implies that we can find an  $A_M(m_0, \frac{m_0}{4}) \subset [M]^{m_0}$ such that $|A_M(m_0, \frac{m_0}{4})| \ge M^{\frac{m_0} 4}$ whose each element is at least $\frac{m_0}{4}$ Hamming distance apart. For such a choice, define the collection of probabilities as  $\cM_r = \{P^\sigma: \sigma \in A_M(m_0, \frac{m_0}{4})\}$ leading to $r \ge M^{\frac{m_0}{4}}$. In the generalized Fano's lemma \ref{lemma:fano} we require $r\ge 2$. To achieve that we simply set $m_0 \ge 8$, as it implies $r \ge M^2 \ge 4$.

Now we find lower bound $\alpha_r$ for the semi-metric and upper bound $\beta_r$ for the Kulback-Leibler divergence. 
Let's start with the upper bound. Since $\text{KL}(P^{\sigma_1}, P^{\sigma_2}) \le C \mu_{k_0}^{-2} h ^{2\gamma_{k_0} + d }\delta(\sigma_1 , \sigma_2)$ for the joint distributions of the dataset $\cD_n$ the  Kulback-Leibler divergence between $\{P^{\sigma_1}\}^{\otimes n}$ and $\{P^{\sigma_2}\}^{\otimes n}$ is upper bounded as:
\[
\begin{aligned}
    & \textstyle  \text{KL}\big(\{P^{\sigma_1}\}^{\otimes n}, \{P^{\sigma_2}\}^{\otimes n}\big) \\
     & = \textstyle n \text{KL}({P^{\sigma_1}}, {P^{\sigma_2}}) \\
     & \le n C \mu_{k_0}^{-2} h ^{2\gamma_{k_0} + d }\delta(\sigma_1 , \sigma_2)\\
     & =\textstyle  n C \mu_{k_0}^{-2} L^{2\gamma_{k_0} + d } \mu_{k_0}^{\frac{2\gamma_{k_0} + d}{\gamma_{k_0}}}  n ^{-\frac{2\gamma_{k_0} + d}{2\gamma_{k_0} + d}} \quad \quad \textstyle  (\text{because} ~ h \text{ is defined as } L \times   \mu_{k_0}^{\frac{1}{\gamma_{k_0}}} n^{-\frac{1}{2\gamma_{k_0} + d}} )\\
     & \le\textstyle   C  L^{2\gamma_{k_0} + d } \mu_{k_0}^{\frac{ d}{\gamma_{k_0}}}  \frac{\log r}{\log M} \quad \quad (\text{because} ~~ r \ge M^{\frac{m_0}{4}})\\
     & \le \textstyle C  L^{2\gamma_{k_0} + d }  \frac{\log r}{\log M} = \beta_r
\end{aligned}
\]
In the Lemma \ref{lemma:fano} we would like $\frac{\beta_r + \log 2}{\log r} \le \frac{3}{4}$ so that we have $1 - \frac{\beta_r + \log 2}{\log r} \ge \frac14$. 
Note that, 
\[
\begin{aligned}
   \textstyle  \frac{\beta_r + \log 2}{\log r} - \frac{3}{4} & \textstyle = \frac{\beta_r}{\log r}  + \frac{\log 2}{\log r} - \frac{3}{4} \\
   & \textstyle = \frac{C L^{2\gamma_{k_0} + d }}{\log M}  + \frac{\log 2}{\log 4} - \frac{3}{4} \quad \quad (\text{because} ~~ r \ge 4,~ \beta_r = C  L^{2\gamma_{k_0} + d }  \frac{\log r}{\log M} )\\
   & = \textstyle \frac{C L^{2\gamma_{k_0} + d }}{\log M} - \frac14 \le 0\\
\end{aligned}
\] for small $L > 0$. We set the $L$ accordingly. Returning to the semi-metric, it is lower bounded as 
\[
\begin{aligned}
     d\big (\theta(P^{\sigma_0}), \theta(P^{\sigma_1})\big)
    &= \textstyle \cE_{P^{\sigma_0}}(g^\star_{\mu, \sigma_1}, \mu)\\
    & \textstyle  \ge \frac{ K_{\gamma, k_0} }{2} \eps^{\gamma_{k_0} + d} \delta(\sigma_0, \sigma_1)\\
   & \textstyle  \ge \frac{ K_{\gamma, k_0} }{2} \eps^{\gamma_{k_0} + d}    \frac{m_0} 4\\
    & \textstyle \ge \frac{ K_{\gamma, k_0} }{8} \eps^{\gamma_{k_0} + d}    K_\alpha 2^{-\alpha}   K_{\gamma, k_0}^\alpha \eps^{ \alpha \gamma_{k_0} -d} \\
    & \quad \quad  (\text{because} ~~ m_0 = \lfloor K_\alpha 2^{-\alpha}   K_{\gamma, k_0}^\alpha \eps^{ \alpha \gamma_{k_0} -d} \rfloor)  \\
    & \textstyle = c_1   \eps^{(1+\alpha) \gamma_{k_0}}\\
    & \textstyle \ge c_2 \big\{\mu_{k_0} n^{-\frac{\gamma_{k_0}}{2\gamma_{k_0} + d}}\big\}^{1 + \alpha}= \alpha_r 
\end{aligned}
\] for some constants $ c_1, c_2 > 0$.
We plug in the lower and upper bound in Fano's lemma \ref{lemma:fano} to obtain the lower bound: 
\[
\textstyle \frac{\alpha_r}{2} \big (1 - \frac{\beta_r + \log 2}{\log r}\big ) \ge \frac{c_2 \big\{\mu_{k_0} n^{-\frac{\gamma_{k_0}}{2\gamma_{k_0} + d}}\big\}^{1 + \alpha}}{2} \times \frac14 \ge c_3 \big\{\mu_{k_0} n^{-\frac{\gamma_{k_0}}{2\gamma_{k_0} + d}}\big\}^{1 + \alpha} 
\] for some $c_3 > 0$ that is independent of both $n$ and $\mu$. 

\end{proof}